\documentclass[12pt]{article}
\usepackage{graphicx}
\usepackage{epstopdf}
\usepackage{placeins}
\usepackage{subfigure}
\usepackage{algpseudocode}
\usepackage{algorithmicx}
\usepackage[boxruled,algosection]{algorithm2e}

\usepackage{url}

\usepackage{amsmath}
\usepackage{amsfonts}
\usepackage{amsthm}
\usepackage{cite}

\usepackage{geometry}

\newtheorem{theorem}{Theorem}
\newtheorem{definition}{Definition}

\usepackage[affil-it]{authblk} 
\usepackage{etoolbox}
\usepackage{lmodern}
\makeatletter
\patchcmd{\@maketitle}{\LARGE \@title}{\fontsize{16}{19.2}\selectfont\@title}{}{}
\makeatother

\title{Query-oriented text summarization based on hypergraph transversals}
\author{Hadrien Van Lierde and Tommy W. S. Chow\\
	Department of Electronic Engineering, City University of Hong Kong\\
	83 Tat Chee Av., Kowloon Tong, Hong Kong, China\\
	hadrien.vanlierde@hotmail.com, eetchow@cityu.edu.hk\\
	}
\date{}

\begin{document}
\footnotesize\noindent\textit{This is the unrefereed Author's Original Version (or pre-print Version) of the article. The present version is not the Accepted Manuscript.}
{\let\newpage\relax\maketitle}
\maketitle

\begin{abstract}
Existing graph- and hypergraph-based algorithms for document summarization represent the sentences of a corpus as the nodes of a graph or a hypergraph in which the edges represent relationships of lexical similarities between sentences. Each sentence of the corpus is then scored individually, using popular node ranking algorithms, and a summary is produced by extracting highly scored sentences. This approach fails to select a subset of \textit{jointly} relevant sentences and it may produce redundant summaries that are missing important topics of the corpus. To alleviate this issue, a new hypergraph-based summarizer is proposed in this paper, in which each node is a sentence and each hyperedge is a theme, namely a group of sentences sharing a topic. Themes are weighted in terms of their prominence in the corpus and their relevance to a user-defined query. It is further shown that the problem of identifying a subset of sentences covering the relevant themes of the corpus is equivalent to that of finding a hypergraph transversal in our theme-based hypergraph. Two extensions of the notion of hypergraph transversal are proposed for the purpose of summarization, and polynomial time algorithms building on the theory of submodular functions are proposed for solving the associated discrete optimization problems. The worst-case time complexity of the proposed algorithms is squared in the number of terms, which makes it cheaper than the existing hypergraph-based methods. A thorough comparative analysis with related models on DUC benchmark datasets demonstrates the effectiveness of our approach, which outperforms existing graph- or hypergraph-based methods by at least $6\%$ of ROUGE-SU4 score.\\

\noindent\textbf{keywords:} Query-Oriented Text Summarization, Hypergraph Theory, Hypergraph Transversal, Sentence Clustering, Submodular Set Functions
\end{abstract}

\section{Introduction}\label{introSection}
The development of automatic tools for the summarization of large corpora of documents has attracted a widespread interest in recent years. With fields of application ranging from medical sciences to finance and legal science, these summarization systems considerably reduce the time required for knowledge acquisition and decision making, by identifying and formatting the relevant information from a collection of documents. Since most applications involve large corpora rather than single documents, summarization systems developed recently are meant to produce summaries of multiple documents. Similarly, the interest has shifted from generic towards query-oriented summarization, in which a query expresses the user's needs. Moreover, existing summarizers are generally extractive, namely they produce summaries by extracting relevant sentences from the original corpus.

Among the existing extractive approaches for text summarization, graph-based methods are considered very effective due to their ability to capture the global patterns of connection between the sentences of the corpus. These systems generally define a graph in which the nodes are the sentences and the edges denote relationships of lexical similarities between the sentences. The sentences are then scored using graph ranking algorithms such as the PageRank \cite{lexrank} or HITS \cite{hits} algorithms, which can also be adapted for the purpose of query-oriented summarization \cite{R17}. A key step of graph-based summarizers is the way the graph is constructed, since it has a strong impact on the sentence scores. As pointed out in \cite{hypersum}, a critical issue of traditional graph-based summarizers is their inability to capture group relationships among sentences since each edge of a graph only connects a pair of nodes. 

Following the idea that each topic of a corpus connects a group of multiple sentences covering that topic, hypergraph models were proposed in \cite{hypersum} and \cite{herf}, in which the hyperedges represent similarity relationships among groups of sentences. These group relationships are formed by detecting clusters of lexically similar sentences we refer to as \textit{themes} or \textit{theme-based hyperedges}. Each theme is believed to cover a specific topic of the corpus. However, since the models of \cite{hypersum} and \cite{herf} define the themes as groups of lexically similar sentences, the underlying topics are not explicitly discovered. Moreover, their themes do not overlap which contradicts the fact that each sentence carries multiple information and may thus belong to multiple themes, as can be seen from the following example of sentence.

\begin{itemize}
\itemsep0em
\item[] "Once John finished studying for his school test the next day, he caught up with his friend at the sport centre and they played soccer together."
\end{itemize}

Two topics are covered by the sentence above: the topics of \textit{studies} and \textit{leisure}. Hence, the sentence should belong to multiple themes simultaneously, which is not allowed in existing hypergraph models of \cite{hypersum} and \cite{herf}. 

The hypergraph model proposed in this paper alleviates these issues by first extracting topics, i.e. groups of semantically related terms, using a new topic model referred to as \textit{SEMCOT}. Then, a theme is associated to each topic, such that each theme is defined a the group of sentences covering the associated topic. Finally, a hypergraph is formed with sentences as nodes, themes as hyperedges and hyperedge weights reflecting the prominence of each theme and its relevance to the query. In such a way, our model alleviates the weaknesses of existing hypergraph models since each theme-based hyperedge is associated to a specific topic and each sentence may belong to multiple themes. 

Furthermore, a common drawback of existing graph- and hypergraph-based summarizers is that they select sentences based on the computation of an individual relevance score for each sentence. This approach fails to capture the information jointly carried by the sentences which results in redundant summaries missing important topics of the corpus. To alleviate this issue, we propose a new approach of sentence selection using our theme-based hypergraph. A minimal hypergraph transversal is the smallest subset of nodes covering all hyperedges of a hypergraph \cite{gunopulous1997}. The concept of hypergraph transversal is used in computational biology \cite{klamt2009} and data mining \cite{gunopulous1997} for identifying a subset of relevant agents in a hypergraph. In the context of our theme-based hypergraph, a hypergraph transversal can be viewed as the smallest subset of sentences covering all themes of the corpus. We extend the notion of transversal to take the theme weights into account and we propose two extensions called \textit{minimal soft hypergraph transversal} and \textit{maximal budgeted hypergraph transversal}. The former corresponds to finding a subset of sentences of minimal aggregated length and achieving a \textit{target coverage} of the topics of the corpus (in a sense that will be clarified). The latter seeks a subset of sentences maximizing the total weight of covered hyperedges while not exceeding a \textit{target summary length}. As the associated discrete optimization problems are NP-hard, we propose two approximation algorithms building on the theory of submodular functions. Our transversal-based approach for sentence selection alleviates the drawback of methods of individual sentence scoring, since it selects a set of sentences that are jointly covering a maximal number of relevant themes and produces informative and non-redundant summaries. As demonstrated in the paper, the time complexity of the method is equivalent to that of early graph-based summarization systems such as LexRank \cite{lexrank}, which makes it more efficient than existing hypergraph-based summarizers \cite{hypersum,herf}. The scalability of summarization algorithms is essential, especially in applications involving large corpora such as the summarization of news reports \cite{hong2014repository} or the summarization of legal texts \cite{kanapala2017text}.

The method of \cite{takamura2009} proposes to select sentences by using a maximum coverage approach, which shares some similarities with our model. However, they attempt to select a subset of sentences maximizing the number of relevant terms covered by the sentences. Hence, they fail to capture the topical relationships among sentences, which are, in contrast, included in our theme-based hypergraph. 

A thorough comparative analysis with state-of-the-art summarization systems is included in the paper. Our model is shown to outperform other models on a benchmark dataset produced by the \textit{Document Understanding Conference}. The main contributions of this paper are (1) a new topic model extracting groups of semantically related terms based on patterns of term co-occurrences, (2) a natural hypergraph model representing nodes as sentences and each hyperedge as a theme, namely a group of sentences sharing a topic, and (3) a new sentence selection approach based on hypergraph transversals for the extraction of a subset of jointly relevant sentences.

The structure of the paper is as follows. In section \ref{relatedWorks}, we present work related to our method. In section \ref{overallSection}, we present an overview of our system which is described in further details in section \ref{mainSection}. Then, in section \ref{experimentSection}, we present experimental results. Finally, section \ref{concludingSection} presents a discussion and concluding remarks.

\section{Background and related work}\label{relatedWorks}
While early models focused on the task of single document summarization, recent systems generally produce summaries of corpora of documents \cite{mckeown2011}. Similarly, the focus has shifted from generic summarization to the more realistic task of query-oriented summarization, in which a summary is produced with the essential information contained in a corpus that is also relevant to a user-defined query \cite{nenkova2012}.

Summarization systems are further divided into two classes, namely abstractive and extractive models. Extractive summarizers identify relevant sentences in the original corpus and produce summaries by aggregating these sentences \cite{mckeown2011}. In contrast, an abstractive summarizer identifies conceptual information in the corpus and reformulates a summary from scratch \cite{nenkova2012}. Since abstractive approaches require advanced natural language processing, the majority of existing summarization systems consist of extractive models. 

Extractive summarizers differ in the method used to identify relevant sentences, which leads to a classification of models as either feature-based or graph-based approaches. Feature-based methods represent the sentences with a set of predefined features such as the sentence position, the sentence length or the presence of cue phrases \cite{fattah2014}. Then, they train a model to compute relevance scores for the sentences based on their features. Since feature-based approaches generally require datasets with labelled sentences which are hard to produce \cite{nenkova2012}, unsupervised graph-based methods have attracted growing interest in recent years. 

Graph-based summarizers represent the sentences of a corpus as the nodes of a graph with the edges modelling relationships of similarity between the sentences \cite{lexrank}. Then, graph-based algorithms are applied to identify relevant sentences. The models generally differ in the type of relationship captured by the graph or in the sentence selection approach. Most graph-based models define the edges connecting sentences based on the co-occurrence of terms in pairs of sentences \cite{lexrank,R17,hypersum}. Then, important sentences are identified either based on node ranking algorithms, or using a global optimization approach. Methods based on node ranking compute individual relevance scores for the sentences and build summaries with highly scored sentences. The earliest such summarizer, LexRank \cite{lexrank}, applies the PageRank algorithm to compute sentence scores. Introducing a query bias in the node ranking algorithm, this method can be adapted for query-oriented summarization as in \cite{R17}. A different graph model was proposed in \cite{zha2002}, where sentences and key phrases form the two classes of nodes of a bipartite graph. The sentences and the key phrases are then scored simultaneously by applying a mutual reinforcement algorithm. An extended bipartite graph ranking algorithm is also proposed in \cite{hits} in which the sentences represent one class of nodes and clusters of similar sentences represent the other class. The hubs and authorities algorithm is then applied to compute sentence scores. Adding terms as a third class of nodes, \cite{zhang2012} propose to score terms, sentences and sentence clusters simultaneously, based on a mutual reinforcement algorithm which propagates the scores across the three node classes. A common drawback of the approaches based on node ranking is that they compute individual relevance scores for the sentences and they fail to model the information jointly carried by the sentences, which may result in redundant summaries. Hence, global optimization approaches were proposed to select a set of jointly relevant and non-redundant sentences as in \cite{bilmes2010} and \cite{wenpeng2015}. For instance, \cite{shen2010} propose a greedy algorithm to find a dominating set of nodes in the sentence graph. A summary is then formed with the corresponding set of sentences. Similarly, \cite{bilmes2010} extract a set of sentences with a maximal similarity with the entire corpus and a minimal pairwise lexical similarity, which is modelled as a multi-objective optimization problem. In contrast, \cite{takamura2009} propose a coverage approach in which a set of sentences maximizing the number of distinct relevant terms is selected. Finally, \cite{wenpeng2015} propose a two step approach in which individual sentence relevance scores are computed first. Then a set of sentences with a maximal total relevance and a minimal joint redundancy is selected. All three methods attempt to solve NP-hard problems. Hence, they propose approximation algorithms based on the theory of submodular functions. 

Going beyond pairwise lexical similarities between sentences and relations based on the co-occurrence of terms, hypergraph models were proposed, in which nodes are sentences and hyperedges model group relationships between sentences \cite{hypersum}. The hyperedges of the hypergraph capture topical relationships among groups of sentences. Existing hypergraph-based systems \cite{hypersum,herf} combine pairwise lexical similarities and clusters of lexically similar sentences to form the hyperedges of the hypergraph. Hypergraph ranking algorithms are then applied to identify important and query-relevant sentences. However, they do not provide any interpretation for the clusters of sentences discovered by their method. Moreover, these clusters do not overlap, which is incoherent with the fact that each sentence carries multiple information and hence belongs to multiple semantic groups of sentences. In contrast, each hyperedge in our proposed hypergraph connects sentences covering the same topic, and these hyperedges do overlap. 

A minimal hypergraph transversal is a subset of the nodes of hypergraph of minimum cardinality and such that each hyperedge of the hypergraph is incident to at least one node in the subset \cite{gunopulous1997}. Theoretically equivalent to the minimum hitting set problem, the problem of finding a minimum hypergraph transversal can be viewed as finding a subset of representative nodes covering the essential information carried by each hyperedge. Hence, hypergraph transversals find applications in various areas such as computational biology, boolean algebra and data mining \cite{andrew2017}. Extensions of hypergraph transversals to include hyperedge and node weights were also proposed in \cite{boros}. Since the associated optimization problems are generally NP-hard, various approximation algorithms were proposed, including greedy algorithms \cite{wolsey1982} and LP relaxations \cite{gainerdewar2017}. The problem of finding a hypergraph transversal is conceptually similar to that of finding a summarizing subset of a set of objects modelled as a hypergraph. However, to the best of our knowledge, there was no attempt to use hypergraph transversals for text summarization in the past. Since it seeks a set of jointly relevant sentences, our method shares some similarities with existing graph-based models that apply global optimization strategies for sentence selection \cite{takamura2009,bilmes2010,wenpeng2015}. However, our hypergraph better captures topical relationships among sentences than the simple graphs based on lexical similarities between sentences.

\section{Problem statement and system overview}\label{overallSection}
Given a corpus of $N_d$ documents and a user-defined query $q$, we intend to produce a summary of the documents with the information that is considered both central in the corpus and relevant to the query. Since we limit ourselves to the production of extracts, our task is to extract a set $S$ of relevant sentences from the corpus and to aggregate them to build a summary. Let $N_s$ be the total number of sentences in the corpus. We further split the task into two subtasks:
\begin{itemize}
\itemsep0em
\item \textbf{target summary length}: the summary must cover the largest amount of relevant information while not exceeding a \text{target length} $L$, namely $\sum_{i\in S}L_i\leq L$, where $\{L_i,1\leq i\leq N_s\}$ represent the lengths of the sentences,
\item \textbf{target coverage}: the summary must have a minimum length while achieving a \textit{target coverage} of the information expressed by a parameter $\gamma\in [0,1]$ expressing the fraction of the information present in the corpus that must be covered by the summary (in a sense that will be clarified).
\end{itemize}

The sentences in the set $S$ are then aggregated to form the final summary. Figure \ref{algochart} summarizes the steps of our proposed method. After some preprocessing steps, the themes are detected based on a topic detection algorithm which tags each sentence with multiple topics. A theme-based hypergraph is then built with the weight of each theme reflecting both its importance in the corpus and its similarity with the query. Finally, depending on the task at hand, one of two types of hypergraph transversal is generated. If the summary must not exceed a \textit{target summary length}, then a \textit{maximal budgeted hypergraph transversal} is generated. If the summary must achieve a \textit{target coverage}, then a \textit{minimal soft hypergraph transversal} is generated. Finally the sentences corresponding to the generated transversal are selected for the summary. 

\begin{figure}[!h]
\centering
\includegraphics[width=.9\textwidth]{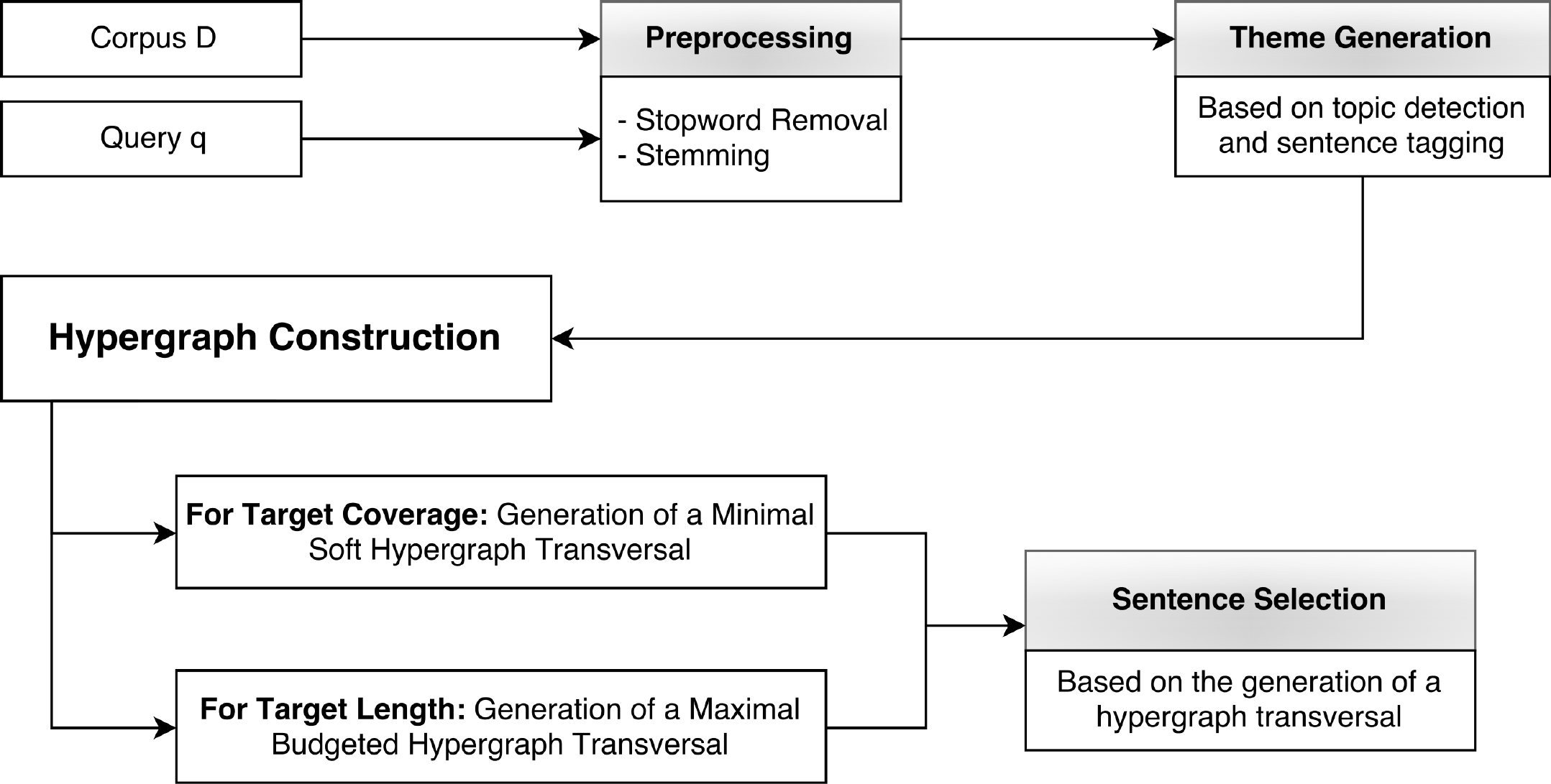}
\caption{Algorithm Chart.}
\label{algochart}
\end{figure}

\section{Summarization based on hypergraph transversals}\label{mainSection}
In this section, we present the key steps of our algorithm: after some standard preprocessing steps, topics of semantically related terms are detected from which themes grouping topically similar sentences are extracted. A hypergraph is then formed based on the sentence themes and sentences are selected based on the detection of a hypergraph transversal.

\subsection{Preprocessing and similarity computation}
As the majority of extractive summarization approaches, our model is based on the representation of sentences as vectors. To reduce the size of the vocabulary, we remove stopwords that do not contribute to the meaning of sentences such as "the" or "a", using a publicly available list of $667$ stopwords \footnote{Stopword Lists by Ranks NL Webmaster Tools, https://www.ranks.nl/stopwords, accessed on 15 November 2017}. The words are also stemmed using Porter Stemmer \cite{stemmer}. Let $N_t$ be the resulting number of distinct terms after these two preprocessing steps are performed. We define the \textit{inverse sentence frequency} $\text{isf}(t)$ \cite{KBS1} as
\begin{equation}
\text{isf}(t)=\log\left(\frac{N_s}{N_s^t}\right)
\end{equation} 
where $N_s^t$ is the number of sentences containing term $t$. This weighting scheme yields higher weights for rare terms which are assumed to contribute more to the semantics of sentences \cite{KBS1}.  Sentence $i$ is then represented by a vector $s_i=[\text{tfisf}(i,1),...,\text{tfisf}(i,N_t)]$ where
\begin{equation}
\text{tfisf}(i,t)=\text{tf}(i,t)\text{isf}(t)
\end{equation}
and $\text{tf}(i,t)$ is the frequency of term $t$ in sentence $i$. Finally, to denote the similarity between two text fragments $a$ and $b$ (which can be sentences, groups of sentences or the query), we use the cosine similarity between the $\text{tfisf}$ representations of $a$ and $b$, as suggested in \cite{R17}:
\begin{equation}\label{similaritDef}
\text{sim}(a,b)=\frac{\sum_t \text{tfisf}(a,t)\text{tfisf}(b,t)}{\sqrt{\sum_t\text{tfisf}(a,t)^2}\sqrt{\sum_t\text{tfisf}(b,t)^2}}
\end{equation}
where $\text{tfisf}(a,t)$ is also defined as the frequency of term $t$ in fragment $a$ multiplied by $\text{isf}(t)$. This similarity measure will be used in section \ref{hypergraphSubsection} to compute the similarity with the query $q$.

\subsection{Sentence theme detection based on topic tagging}\label{themeSubsection}
As mentioned in section \ref{introSection}, our hypergraph model is based on the detection of themes. A theme is defined as a group of sentences covering the same topic. Hence, our theme detection algorithm is based on a $3$-step approach: the extraction of topics, the process of tagging each sentence with multiple topics and the detection of themes based on topic tags.

A topic is viewed as a set of semantically similar terms, namely terms that refer to the same subject or the same piece of information. In the context of a specific corpus of related documents, a topic can be defined as a set of terms that are likely to occur close to each other in a document \cite{moitra2012}. In order to extract topics, we make use of a clustering approach based on the definition of a semantic dissimilarity between terms. For terms $u$ and $v$, we first define the joint $\text{isf}$ weight $\text{isf}(u,v)$ as
\begin{equation}
\text{isf}(u,v)=\log\left(\frac{N_s}{N_s^{uv}}\right)
\end{equation}
where $N_s^{uv}$ is the number of sentences in which both terms $u$ and $v$ occur together. Then, the semantic dissimilarity $d_{\text{sem}}(u,v)$ between the two terms is defined as
\begin{equation}\label{semDistEqn}
d_{\text{sem}}(u,v)=\frac{\text{isf}(u,v)-\min(\text{isf}(u),\text{isf}(v))}{\max(\text{isf}(u),\text{isf}(v))}
\end{equation}
which can be viewed as a special case of the so-called google distance which was already successfully applied to learn semantic similarities between terms on webpages \cite{cilibrasi2007}. Using concepts from information theory, $\text{isf}(u)$ represents the number of bits required to express the occurrence of term $u$ in a sentence using an optimally efficient code. Then, $\text{isf}(u,v)-\text{isf}(u)$ can be viewed as the number of bits of information in $v$ relative to $u$. Assuming $\text{isf}(v)\geq\text{isf}(u)$, $d_{\text{sem}}(u,v)$ can be viewed as the improvement obtained when compressing $v$ using a previously compressed code for $u$ and compressing $v$ from scratch \cite{cilibrasi2005clustering}. More details can be found in \cite{cilibrasi2007}. In practice, two terms $u$ and $v$ with a low value of $d_{\text{sem}}(u,v)$ are expected to consistently occur together in the same context, and they are thus considered to be semantically related in the context of the corpus.

Based on the semantic dissimilarity measure between terms, we define a topic as a group of terms with a high semantic density, namely a group of terms such that each term of the group is semantically related to a sufficiently high number of terms in the group. The DBSCAN algorithm is a method of density-based clustering that achieves this result by iteratively growing cohesive groups of agents, with the condition that each member of a group should contain a sufficient number of other members in an $\epsilon$-neighborhood around it \cite{dbscan}. Using the semantic dissimilarity as a distance measure, DBSCAN extracts groups of semantically related terms which are considered as topics. The advantages offered by DBSCAN over other clustering algorithms are threefold. First, DBSCAN is capable of detecting the number of clusters automatically. Second, although the semantic dissimilarity is symmetric and nonnegative, it does not satisfy the triangle inequality. This prevents the use of various clustering algorithms such as the agglomerative clustering with complete linkage \cite{maimon2005}. However, DBSCAN does not explicitly require the triangle inequality to be satisfied. Finally, it is able to detect noisy samples in low density region, that do not belong to any other cluster. 

Given a set of pairwise dissimilarity measures, a density threshold $\epsilon$ and a minimum neighborhood size $m$, DBSCAN returns a number $K$ of clusters and a set of labels $\{c(i)\in\{-1,1,...,K\}:1\leq i\leq N_t\}$ such that $c(i)=-1$ if term $i$ is considered a noisy term. While it is easy to determine a natural value for $m$, choosing a value for $\epsilon$ is not straightforward. Hence, we adapt DBSCAN algorithm to build our topic model referred to as \textit{Semantic Clustering Of Terms} (\textit{SEMCOT}) algorithm. It iteratively applies DBSCAN and decreases the parameter $\epsilon$ until the size of each cluster does not exceed a predefined value. Algorithm \ref{SEMCOTAlgo} summarizes the process. Apart from $m$, the algorithm also takes parameters $\epsilon_0$ (the initial value of $\epsilon$), $M$ (the maximum number of points allowed in a cluster) and $\beta\leq 1$ (a factor close to $1$ by which $\epsilon$ is multiplied until all clusters have sizes lower than $M$). Experiments on real-world data suggest empirical values of $m=3$, $\epsilon_0=0.9$, $M=0.1N_t$ and $\beta=0.95$. Additionally, we observe that, among the terms considered as noisy by DBSCAN, some could be highly infrequent terms with a high $\text{isf}$ value but yet having a strong impact on the meaning of sentences. Hence, we include them as topics consisting of single terms if their $\text{isf}$ value exceeds a threshold $\mu$ whose value is determined by cross-validation, as explained in section \ref{experimentSection}.\\

\begin{algorithm}[H]
INPUT: Semantic Dissimilarities $\{d_{\text{sem}}(u,v):1\leq u,v\leq N_t\}$,\\
PARAMETERS: $\epsilon_0$, $M$, $m$, $\beta\leq 1$, $\mu$\\
OUTPUT: Number $K$ of topics, topic tags $\{c(i):1\leq i\leq N_t\}$\\
$\epsilon\leftarrow\epsilon_0$, $\text{minTerms}\leftarrow m$, $\text{proceed}\leftarrow \text{True}$\\
\textbf{while} $\text{proceed}$:\\
\Indp $[c,K]\leftarrow DBSCAN(d_{\text{sem}},\epsilon,\text{minTerms})$\\
\textbf{if} $\underset{1\leq k\leq K}{\max}(|\{i:c(i)=k\}|)<M$: $\text{proceed}\leftarrow \text{False}$\\
\textbf{else}: $\epsilon\leftarrow \beta\epsilon$\\
\Indm\textbf{for each} $t$ s.t. $c(t)=-1$ (noisy terms):\\
\Indp \textbf{if} $\text{isf}(t)\geq \mu$:\\
\Indp $c(t)\leftarrow K+1$, $K\leftarrow K+1$\\
\caption{SEMCOT}
\label{SEMCOTAlgo}
\end{algorithm}~\\

Once the topics are obtained based on algorithm \ref{SEMCOTAlgo}, a \textit{theme} is associated to each topic, namely a group of sentences covering the same topic. The sentences are first tagged with multiple topics based on a scoring function. The score of the $l$-th topic in the $i$-th sentence is given by 
\begin{equation}
\sigma_{il}=\underset{t:c(t)=l}{\sum}\text{tfisf}(i,t)
\end{equation}
and the sentence is tagged with topic $l$ whenever $\sigma_{il}\geq\delta$, in which $\delta$ is a parameter whose value is tuned as explained in section \ref{experimentSection} (ensuring that each sentence is tagged with at least one topic). The scores are intentionally not normalized to avoid tagging short sentences with an excessive number of topics. The $l$-th theme is then defined as the set of sentences
\begin{equation}
T_l=\{i:\sigma_{il}\geq\delta,1\leq i\leq N_s\}.
\end{equation}

While there exist other summarization models based on the detection of clusters or groups of similar sentence, the novelty of our theme model is twofold. First, each theme is easily interpretable as the set of sentences associated to a specific topic. As such, our themes can be considered as groups of semantically related sentences. Second, it is clear that the themes discovered by our approach do overlap since a single sentence may be tagged with multiple topics. To the best of our knowledge, none of the previous cluster-based summarizers involved overlapping groups of sentences. Our model is thus more realistic since it better captures the multiplicity of the information covered by each sentence.

\subsection{Sentence hypergraph construction}\label{hypergraphSubsection}
A hypergraph is a generalization of a graph in which the hyperedges may contain any number of nodes, as expressed in definition \ref{hypergraphDef} \cite{hypersum}. Our hypergraph model moreover includes both hyperedge and node weights.

\begin{definition}[Hypergraph]\label{hypergraphDef}
A node- and hyperedge-weighted hypergraph is defined as a quadruplet $H=(V,E,\phi,w)$ in which $V$ is a set of nodes, $E\subseteq 2^{V}$ is a set of hyperedges, $\phi\in\mathbb{R}_+^{|V|}$ is a vector of positive node weights and $w\in\mathbb{R}_+^{|E|}$ is a vector of positive hyperedge weights.
\end{definition}

For convenience, we will refer to a hypergraph by its weight vectors $\phi$ and $w$, its hyperedges represented by a set $E\subseteq 2^V$ and its incidence lists $\text{inc}(i)=\{e\in E:i\in e\}$ for each $i\in V$.

As mentioned in section \ref{introSection}, our system relies on the definition of a theme-based hypergraph which models groups of semantically related sentences as hyperedges. Hence, compared to traditional graph-based summarizers, the hypergraph is able to capture more complex group relationships between sentences instead of being restricted to pairwise relationships.

In our sentence-based hypergraph, the sentences are the nodes and each theme defines a hyperedge connecting the associated sentences. The weight $\phi_i$ of node $i$ is the length of the $i$-th sentence, namely:
\begin{equation}
\begin{array}{l}
V = \{1,...,N_s\}\text{ and }\phi_i=L_i\text{, }\text{ }1\leq i\leq N_s\\
E = \{e_1,...,e_K\}\subseteq 2^V\\
e_l=T_l\text{ i.e. }e_l\in\text{inc}(i)\leftrightarrow i\in T_l
\end{array}
\end{equation}

Finally, the weights of the hyperedges are computed based on the centrality of the associated theme and its similarity with the query:
\begin{equation}\label{eqnHyperWeights}
w_l=(1-\lambda)\text{sim}(T_l,D)+\lambda\text{sim}(T_l,q)
\end{equation}
where $\lambda\in [0,1]$ is a parameter and $D$ represents the entire corpus. $\text{sim}(T_l,D)$ denotes the similarity of the set of sentences in theme $T_l$ with the entire corpus (using the tfisf-based similarity of equation \ref{similaritDef}) which measures the centrality of the theme in the corpus. $\text{sim}(T_l,q)$ refers to the similarity of the theme with the user-defined query $q$.

\subsection{Detection of hypergraph transversals for text summarization}
The sentences to be included in the query-oriented summary should contain the essential information in the corpus, they should be relevant to the query and, whenever required, they should either not exceed a \textit{target length} or jointly achieve a \textit{target coverage} (as mentioned in section \ref{overallSection}). Existing systems of graph-based summarization generally solve the problem by ranking sentences in terms of their \textit{individual} relevance \cite{lexrank,R17,hypersum}. Then, they extract a set of sentences with a maximal total relevance and pairwise similarities not exceeding a predefined threshold. However, we argue that the \textit{joint} relevance of a group of sentences is not reflected by the individual relevance of each sentence. And limiting the redundancy of selected sentences as done in \cite{hypersum} does not guarantee that the sentences jointly cover the relevant themes of the corpus.

Considering each topic as a distinct piece of information in the corpus, an alternative approach is to select the smallest subset of sentences covering each of the topics. The latter condition can be reformulated as ensuring that each theme has at least one of its sentences appearing in the summary. Using our sentence hypergraph representation, this corresponds to the detection of a minimal hypergraph transversal as defined below \cite{gunopulous1997}.

\begin{definition}\label{defHypergraphTransv}
Given an unweighted hypergraph $H=(V,E)$, a minimal hypergraph transversal is a subset $S^*\subseteq V$ of nodes satisfying
\begin{equation}
\begin{array}{rcl}
S^*&=&\underset{S\subseteq V}{\text{argmin}}|S|\\
&& \text{s.t. }\underset{i\in S}{\bigcup}\text{inc}(i)=E
\end{array}
\end{equation}
where $\text{inc}(i)=\{e:i\in e\}$ denotes the set of hyperedges incident to node $i$.
\end{definition}

Figure \ref{transversalExampleA} shows an example of hypergraph and a minimal hypergraph transversal of it (star-shaped nodes). In this case, since the nodes and the hyperedges are unweighted, the minimal transversal is not unique.

\begin{figure}[!h]
\centering
\includegraphics[width=.9\textwidth]{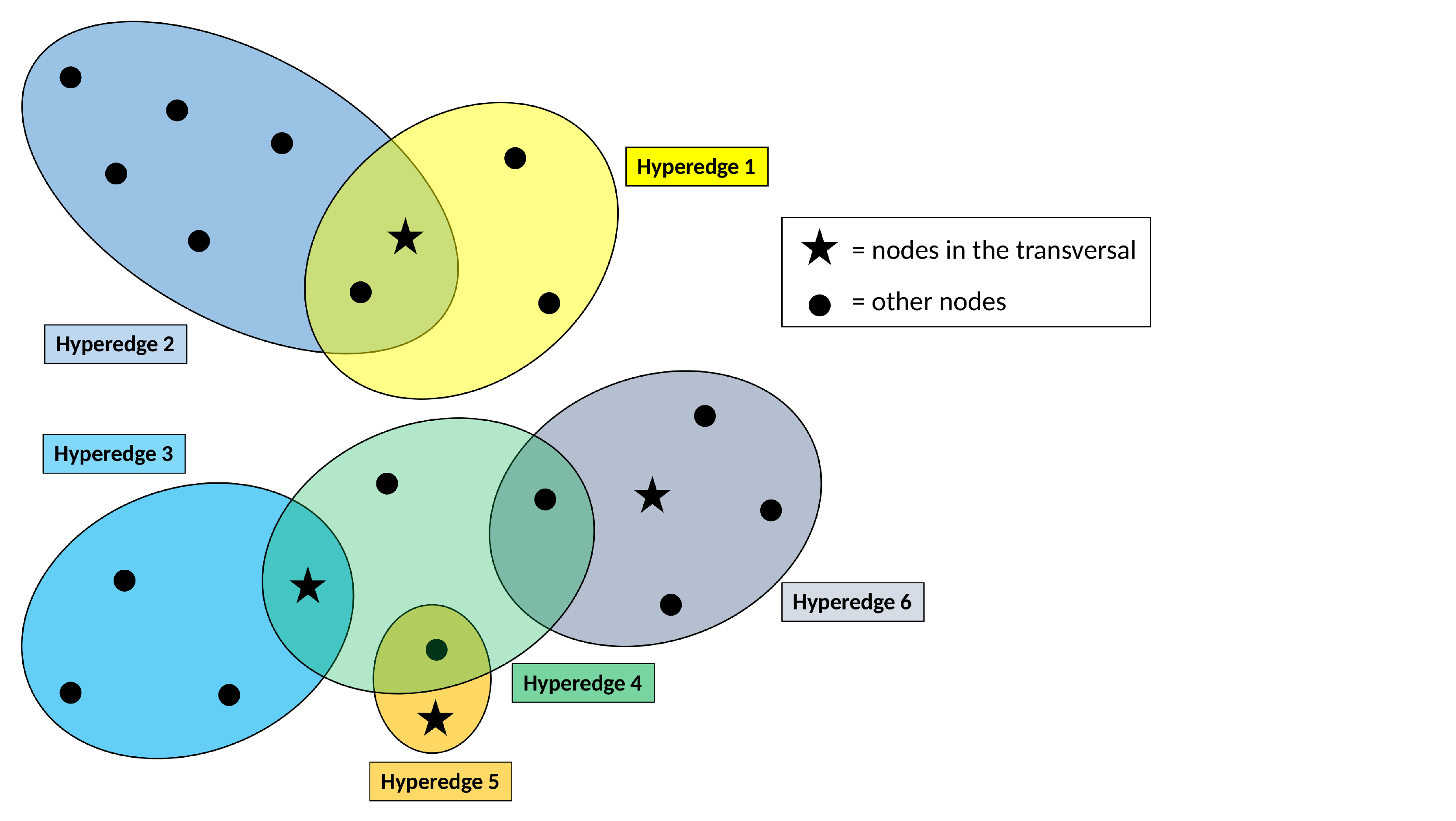}
\caption{Example of hypergraph and minimal hypergraph transversal.}
\label{transversalExampleA}
\end{figure}
 
The problem of finding a minimal transversal in a hypergraph is NP-hard \cite{karp1972}. However, greedy algorithms or LP relaxations provide good approximate solutions in practice \cite{gainerdewar2017}. As intended, the definition of transversal includes the notion of \textit{joint} coverage of the themes by the sentences. However, it neglects node and hyperedge weights and it is unable to identify query-relevant themes. Since both the sentence lengths and the relevance of themes should be taken into account in the summary generation, we introduce two extensions of transversal, namely the \textit{minimal soft hypergraph transversal} and the \textit{maximal budgeted hypergraph transversal}. A minimal soft transversal of a hypergraph is obtained by minimizing the total weights of selected nodes while ensuring that the total weight of covered hyperedges exceeds a given threshold.

\begin{definition}[minimal soft hypergraph transversal]
Given a node and hyperedge weighted hypergraph $H=(V,E,\phi,w)$ and a parameter $\gamma\in [0,1]$, a minimal soft hypergraph transversal is a subset $S^*\subseteq V$ of nodes satisfying
\begin{equation}
\begin{array}{rcl}
S^*&=&\underset{S\subseteq V}{\text{argmin}}\underset{i\in S}{\sum}\phi_i\\
&& \text{s.t. }\underset{e\in \text{inc}(S)}{\sum}w_e\geq\gamma W
\end{array}
\end{equation}
in which $\text{inc}(S)=\underset{i\in S}{\bigcup}\text{inc}(i)$ and $W=\sum_ew_e$. 
\end{definition}

The extraction of a minimal soft hypergraph transversal of the sentence hypergraph produces a summary of minimal length achieving a \textit{target coverage} expressed by parameter $\gamma\in [0,1]$. As mentioned in section \ref{overallSection}, applications of text summarization may also involve a hard constraint on the total summary length $L$. For that purpose, we introduce the notion of \textit{maximal budgeted hypergraph transversal} which maximizes the volume of covered hyperedges while not exceeding the target length.

\begin{definition}[maximal budgeted hypergraph transversal]
Given a node and hyperedge weighted hypergraph $H=(V,E,\phi,w)$ and a parameter $L>0$, a maximal budgeted hypergraph transversal is a subset $S^*\subseteq V$ of nodes satisfying
\begin{equation}\label{defBudgTransv}
\begin{array}{rcl}
S^*&=&\underset{S\subseteq V}{\text{argmax}}\underset{e\in \text{inc}(S)}{\sum}w_e\\
&& \text{s.t. }\underset{i\in S}{\sum}\phi_i\leq L.
\end{array}
\end{equation}
\end{definition} 

We refer to the function $\underset{e\in \text{inc}(S)}{\sum}w_e$ as the \textit{hyperedge coverage} of set $S$. We observe that both weighted transversals defined above include the notion of \textit{joint} coverage of the hyperedges by the selected nodes. As a result and from the definition of hyperedge weights (equation \ref{eqnHyperWeights}), the resulting summary covers themes that are both central in the corpus and relevant to the query. This approach also implies that the resulting summary does not contain redundant sentences covering the exact same themes. As a result selected sentences are expected to cover different themes and to be semantically diverse. Both the problems of finding a minimal soft transversal or finding a maximal budgeted transversal are NP-hard as proved by theorem \ref{npHardTheorem}. 

\begin{theorem}[NP-hardness]\label{npHardTheorem}
The problems of finding a minimal soft hypergraph transversal or a maximal budgeted hypergraph transversal in a weighted hypergraph are NP-hard.
\end{theorem}
\begin{proof}
Regarding the minimal soft hypergraph transversal problem, with parameter $\gamma=1$ and unit node weights, the problem is equivalent to the classical set cover problem (definition \ref{defHypergraphTransv}) which is NP-complete \cite{karp1972}. The maximal budgeted hypergraph transversal problem can be shown to be equivalent to the maximum coverage problem with knapsack constraint which was shown to be NP-complete in \cite{karp1972}.
\end{proof}

Since both problems are NP-hard, we formulate polynomial time algorithms to find approximate solutions to them and we provide the associated approximation factors. The algorithms build on the submodularity and the non-decreasing properties of the hyperedge coverage function, which are defined below.

\begin{definition}[Submodular and non-decreasing set functions]\label{defSubmNonDecr}
Given a finite set $A$, a function $f:2^{A}\rightarrow\mathbb{R}$ is monotonically non-decreasing if $\forall S\subset A$ and $\forall u\in A\setminus S$, 
\begin{equation}
f(S\cup \{u\})\geq f(S)
\end{equation}
and it is submodular if $\forall S,T$ with $S\subseteq T\subset A$, and $\forall u\in A\setminus T$, 
\begin{equation}
f(T\cup \{u\})-f(T)\leq f(S\cup \{u\})-f(S).
\end{equation}
\end{definition}

Based on definition \ref{defSubmNonDecr}, we prove in theorem \ref{thmSubm} that the hyperedge coverage function is submodular and monotonically non-decreasing, which provides the basis of our algorithms.

\begin{theorem}\label{thmSubm}
Given a hypergraph $H=(V,E,\phi,w)$, the hyperedge coverage function $f:2^V\rightarrow\mathbb{R}$ defined by 
\begin{equation}
f(S)=\underset{e\in \text{inc}(S)}{\sum}w_e
\end{equation}
is submodular and monotonically non-decreasing.
\end{theorem}

\begin{proof}
The hyperege coverage function $f$ is clearly monotonically non-decreasing and it is submodular since $\forall S\subseteq T\subset V$, and $s\in V\setminus T$, 
\begin{equation}
\begin{array}{l}
(f(S\cup\{s\})-f(S))-(f(T\cup\{s\})-f(T))\\
=\left[\underset{e\in\text{inc}(S\cup \{s\})}{\sum}w_e-\underset{e\in\text{inc}(S)}{\sum}w_e\right]-\left[\underset{e\in\text{inc}(T\cup \{s\})}{\sum}w_e-\underset{e\in\text{inc}(T)}{\sum}w_e\right]\\
= \left[ \underset{e\in \text{inc}(\{s\})\setminus\text{inc}(S)}{\sum}w_e\right]-\left[ \underset{e\in \text{inc}(\{s\})\setminus\text{inc}(T)}{\sum}w_e\right]\\
= \underset{e\in (\text{inc}(T)\cap\text{inc}(\{s\}))\setminus\text{inc}(S)}{\sum}w_e\geq 0
\end{array}
\end{equation}
where $\text{inc}(R)=\{e:e\cap S\neq\emptyset\}$ for $R\subseteq V$. The last equality follows from $\text{inc}(S)\subseteq \text{inc}(T)$ and $\text{inc}(\{s\})\setminus\text{inc}(T)\subseteq\text{inc}(\{s\})\setminus\text{inc}(S)$.
\end{proof}

Various classes of NP-hard problems involving a submodular and non-decreasing function can be solved approximately by polynomial time algorithms with provable approximation factors. Algorithms \ref{budgTransAlgo} and \ref{softTransAlgo} are our core methods for the detection of approximations of maximal budgeted hypergraph transversals and minimal soft hypergraph transversals, respectively. In each case, a transversal is found and the summary is formed by extracting and aggregating the associated sentences. Algorithm \ref{budgTransAlgo} is based on an adaptation of an algorithm presented in \cite{leskovec2007} for the maximization of submodular functions under a Knaspack constraint. It is our primary transversal-based summarization model, and we refer to it as the method of \textit{Transversal Summarization with Target Length} (\textit{TL-TranSum} algorithm). Algorithm \ref{softTransAlgo} is an application of the algorithm presented in \cite{wolsey1982} for solving the submodular set covering problem. We refer to it as \textit{Transversal Summarization with Target Coverage} (\textit{TC-TranSum} algorithm). Both algorithms produce transversals by iteratively appending the node inducing the largest increase in the total weight of the covered hyperedges relative to the node weight. While long sentences are expected to cover more themes and induce a larger increase in the total weight of covered hyperedges, the division by the node weights (i.e. the sentence lengths) balances this tendency and allows the inclusion of short sentences as well. In contrast, the methods of sentence selection based on a maximal relevance and a minimal redundancy such as, for instance, the maximal marginal relevance approach of \cite{carbonell1998}, tend to favor the selection of long sentences only. The main difference between algorithms \ref{budgTransAlgo} and \ref{softTransAlgo} is the stopping criterion: in algorithm \ref{softTransAlgo}, the approximate minimal soft transversal is obtained whenever the targeted hyperedge coverage is reached while algorithm \ref{budgTransAlgo} appends a given sentence to the approximate maximal budgeted transversal only if its addition does not make the summary length exceed the target length $L$.\\

\begin{algorithm}[H]
INPUT: Sentence Hypergraph $H=(V,E,\phi,w)$, target length $L$.\\
OUTPUT: Set $S$ of sentences to be included in the summary.\\
\textbf{for each} $i\in\{1,...,N_s\}$: $r_i\leftarrow \frac{1}{\phi_i}\underset{e\in\text{inc}(i)}{\sum}w_e$\\
$R\leftarrow\emptyset$, $Q\leftarrow V$, $f\leftarrow 0$\\
\textbf{while} $Q\neq\emptyset$:\\
\Indp $s^*\leftarrow \underset{i\in Q}{\text{argmax}}\text{ }r_i$, $Q\leftarrow Q\setminus\{s^*\}$\\
\textbf{if} $\phi_{s^*}+f\leq L$:\\
\Indp $R\leftarrow R\cup\{s^*\}$, $f\leftarrow f+l^*$\\
\textbf{for each} $i\in \{1,...,N_s\}$: $r_i\leftarrow r_i-\frac{\underset{e\in\text{inc}(s^*)\cap\text{inc}(i)}{\sum} w_e}{\phi_i}$\\
\Indm \Indm Let $G\leftarrow \{\{i\}\text{ : }i\in V,\phi_i\leq L\}$\\
$S\leftarrow \underset{S\in\{Q\}\cup G}{\text{argmax}}\text{ }\text{ }\text{ }\underset{e\in\text{inc}(S)}{\sum}w_e$\\
return $S$\\
\caption{\textit{Transversal Summarization with Target Length} (\textit{TL-TranSum})}
\label{budgTransAlgo}
\end{algorithm}~\\

\begin{algorithm}[H]
INPUT: Sentence Hypergraph $H=(V,E,\phi,w)$, parameter $\gamma\in [0,1]$.\\
OUTPUT: Set $S$ of sentences to be included in the summary.\\
\textbf{for each} $i\in\{1,...,N_s\}$: $r_i\leftarrow \frac{1}{\phi_i}\underset{e\in\text{inc}(i)}{\sum}w_e$\\
$S\leftarrow\emptyset$, $Q\leftarrow V$, $\tilde{W}\leftarrow 0$, $W\leftarrow \sum_ew_e$\\
\textbf{while} $Q\neq\emptyset$ and $\tilde{W}<\gamma W$:\\
\Indp $s^*\leftarrow \underset{i\in Q}{\text{argmax}}\text{ }r_i$\\
$S\leftarrow S\cup\{s^*\}$, $\tilde{W}\leftarrow \tilde{W}+\phi_{s*}r_{s^*}$\\
\textbf{for each} $i\in \{1,...,N_s\}$: $r_i\leftarrow r_i-\frac{\underset{e\in\text{inc}(s^*)\cap\text{inc}(i)}{\sum} w_e}{\phi_i}$\\
\Indm return $S$\\
\caption{\textit{Transversal Summarization with Target Coverage} (\textit{TC-TranSum})}
\label{softTransAlgo}
\end{algorithm}~\\

We next provide theoretical guarantees that support the formulation of algorithms \ref{budgTransAlgo} and \ref{softTransAlgo} as approximation algorithms for our hypergraph transversals. Theorem \ref{thmGuaranteeBudg} provides a constant approximation factor for the output of algorithm \ref{budgTransAlgo} for the detection of minimal soft hypergraph transversals. It builds on the submodularity and the non-decreasing property of the hyperedge coverage function.

\begin{theorem}\label{thmGuaranteeBudg}
Let $S^L$ be the summary produced by our TL-TranSum algorithm \ref{budgTransAlgo}, and $S^*$ be a maximal budgeted transversal associated to the sentence hypergraph, then
\begin{equation}
\underset{e\in \text{inc}(S^L)}{\sum}w_e \geq \frac{1}{2}\left(1-\frac{1}{e}\right)\underset{e\in \text{inc}(S^*)}{\sum}w_e.
\end{equation}
\end{theorem}
\begin{proof}
Since the hyperedge coverage function is submodular and monotonically non-decreasing, the extraction of a maximal budgeted transversal is a problem of maximization of a submodular and monotonically non-decreasing function under a Knapsack constraint, namely
\begin{equation}
\underset{S\subseteq V}{\max}f(S)\text{ s.t. }\underset{i\in S}{\sum}\phi_i\leq L
\end{equation}
where $f(S)=\underset{e\in \text{inc}(S)}{\sum}w_e$. Hence, by theorem 2 in \cite{leskovec2007}, the algorithm forming a transversal $S^F$ by iteratively growing a set $S_t$ of sentences according to
\begin{equation}\label{iterativeAlgobudg}
S_{t+1}=S_t\cup\left\lbrace \underset{s\in V\setminus S_t}{\text{argmax}}\left\lbrace\frac{f(S\cup\{s\})-f(S)}{\phi_s}, \phi_s+\underset{i\in S_t}{\sum}\phi_i\leq L\right\rbrace\right\rbrace
\end{equation}
produces a final summary $S^F$ satisfying
\begin{equation}
f(S^F)\geq f(S^*)\frac{1}{2}\left(1-\frac{1}{e}\right).
\end{equation}
As algorithm \ref{budgTransAlgo} implements the iterations expressed by equation \ref{iterativeAlgobudg}, it achieves a constant approximation factor of $\frac{1}{2}\left(1-\frac{1}{e}\right)$.
\end{proof}

Similarly, theorem \ref{thmGuaranteeSoft} provides a data-dependent approximation factor for the output of algorithm \ref{softTransAlgo} for the detection of maximal budgeted hypergraph transversals. It also builds on the submodularity and the non-decreasing property of the hyperedge coverage function.

\begin{theorem}\label{thmGuaranteeSoft}
Let $S^P$ be the summary produced by our TC-TranSum algorithm \ref{softTransAlgo} and let $S^*$ be a minimal soft hypergraph transversal, then
\begin{equation}
\underset{i\in S^P}{\sum}\phi_i\leq \underset{i\in S^*}{\sum}\phi_i \left(1+\log\left(\frac{\gamma W}{\gamma W-\underset{e\in \text{inc}(S^{T-1})}{\sum}w_e}\right)\right)
\end{equation}
where $S_1,...,S_T$ represent the consecutive sets of sentences produced by algorithm \ref{softTransAlgo}. 
\end{theorem}

\begin{proof}
Consider the function $g(S)=\min(\gamma W,\underset{e\in\text{inc}(S)}{\sum}w_e)$. Then the problem of finding a minimal soft hypergraph transversal can be reformulated as
\begin{equation}
S^*=\underset{S\subseteq V}{\text{argmin}} \underset{s\in S}{\sum}\phi_s\text{ s.t. }g(S)\geq g(V)
\end{equation}
As $g$ is submodular and monotonically non-decreasing, theorem 1 in \cite{wolsey1982} shows that the summary $S^G$ produced by iteratively growing a set $S_t$ of sentences such that 
\begin{equation}\label{iterativeAlgosoft}
S_{t+1}=S_t\cup\left\lbrace \underset{s\in V\setminus S_t}{\text{argmax}}\left\lbrace\frac{f(S\cup\{s\})-f(S)}{\phi_s}\right\rbrace\right\rbrace
\end{equation}
produces a summary $S^G$ satisfying
\begin{equation}
\underset{i\in S^G}{\sum}\phi_i\leq \underset{i\in S^*}{\sum}\phi_i \left(1+\log\left(\frac{g(V)}{g(V)-g(S^{T-1})}\right)\right).
\end{equation}
which can be rewritten as
\begin{equation}
\underset{i\in S^G}{\sum}\phi_i\leq \underset{i\in S^*}{\sum}\phi_i \left(1+\log\left(\frac{\gamma W}{\gamma W-\underset{e\in \text{inc}(S^{T-1})}{\sum}w_e}\right)\right).
\end{equation}
As algorithm \ref{softTransAlgo} implements the iterations expressed by equation \ref{iterativeAlgosoft}, the summary $S^S$ produced by our algorithm \ref{softTransAlgo} satisfies the same inequality.
\end{proof}

In practice, the result of theorem \ref{thmGuaranteeSoft} suggests that the quality of the output depends on the relative increase in the hyperedge coverage induced by the last sentence to be appended to the summary. In particular, if each sentence that is appended to the summary in the interations of algorithm \ref{softTransAlgo} covers a sufficient number of new themes that are not covered already by the summary, the approximation factor is low.

\subsection{Complexity analysis}
We analyse the worst case time complexity of each step of our method. The time complexity of DBSCAN algorithm \cite{dbscan} is $O(N_t\log(N_t))$. Hence, the theme detection algorithm \ref{SEMCOTAlgo} takes $O(N_cN_t\log(N_t))$ steps, where $N_c$ is the number of iterations of algorithm \ref{SEMCOTAlgo} which is generally low compared to the number of terms. The time complexity for the hypergraph construction is $O(K(N_s+N_t))$ where $K$ is the number of topics, or $O(N_t^2)$ if $N_t\geq N_s$. The time complexity of the sentence selection algorithms \ref{budgTransAlgo} and \ref{softTransAlgo} are bounded by $O(N_sKC^{\max}L^{\max})$ where $C^{\max}$ is the number of sentences in the largest theme and $L^{\max}$ is the length of the longest sentences. Assuming $N_t$ is larger than $N_s$, the overall time complexity of the method is of $O(N_t^2)$ steps in the worst case. Hence the method is essentially equivalent to early graph-based models for text summarization in terms of computational burden, such as the LexRank-based systems \cite{lexrank,R17} or greedy approaches based on global optimization \cite{shen2010,bilmes2010,wenpeng2015}. However, it is computationnally more efficient than traditional hypergraph-based summarizers such as the one in \cite{herf} which involves a Markov Chain Monte Carlo inference for its topic model or the one in \cite{hypersum} which is based on an iterative computation of scores involving costly matrix multiplications at each step.

\section{Experiments and evaluation}\label{experimentSection}
We present experimental results obtained with a Python implementation of algorithms \ref{budgTransAlgo} and \ref{softTransAlgo} on a standard computer with a $2.5GHz$ processor and a 8GB memory.

\subsection{Dataset and metrics for evaluation}
We test our algorithms on DUC2005 \cite{duc2005}, DUC2006 \cite{duc2006} and DUC2007 \cite{duc2007} datasets which were produced by the Document Understanding Conference (DUC) and are widely used as benchmark datasets for the evaluation of query-oriented summarizers. The datasets consist respectively of $50$, $50$ and $45$ corpora, each consisting of $25$ documents of approximately $1000$ words, on average. A query is associated to each corpus. For evaluation purposes, each corpus is associated with a set of query-relevant summaries written by humans called \textit{reference summaries}. In each of our experiments, a candidate summary is produced for each corpus by one of our algorithms and it is compared with the reference summaries using the metrics described below. Moreover, in experiments involving algorithm \ref{budgTransAlgo}, the target summary length is set to $250$ words as required in DUC evalutions.

In order to evaluate the similarity of a candidate summary with a set of reference summaries, we make use of the ROUGE toolkit of \cite{rouge}, and more specifically of ROUGE-2 and ROUGE-SU4 metrics, which were adopted by DUC for summary evaluation. ROUGE-2 measures the number of bigrams found both in the candidate summary and the set of reference summaries. ROUGE-SU4 extends this approach by counting the number of unigrams and the number of 4-skip-bigrams appearing in the candidate and the reference summaries, where a 4-skip-bigram is a pair of words that are separated by no more than 4 words in a text. We refer to ROUGE toolkit \cite{rouge} for more details about the evaluation metrics. ROUGE-2 and ROUGE-SU4 metrics are computed following the same setting as in DUC evaluations, namely with word stemming and jackknife resampling but without stopword removal.

\subsection{Parameter tuning}
Besides the parameters of SEMCOT algorithm for which empirical values were given in section \ref{themeSubsection}, there are three parameters of our system that need to be tuned: parameters $\mu$ (threshold on isf value to include a noisy term as a single topic in SEMCOT), $\delta$ (threshold on the topic score for tagging a sentence with a given topic) and $\lambda$ (balance between the query relevance and the centrality in hyperedge weights). The values of all three parameters are determined by an alternating maximization strategy of ROUGE-SU4 score in which the values of two parameters are fixed and the value of the third parameter is tuned to maximize the ROUGE-SU4 score produced by algorithm \ref{budgTransAlgo} with a target summary length of $250$ words, in an iterative fashion. The ROUGE-SU4 scores are evaluated by cross-validation using a leave-one-out process on a validation dataset consisting of $70\%$ of DUC2007 dataset, which yields $\mu=1.98$, $\delta=0.85$ and $\lambda=0.4$.

Additionally, we display the evolution of ROUGE-SU4 and ROUGE-2 scores as a function of $\delta$ and $\lambda$. For parameter $\delta$, we observe in graphs \ref{deltaRouge2} and \ref{deltaRougeSU4} that the quality of the summary is low for $\delta$ close to $0$ since it encourages our theme detection algorithm to tag the sentences with irrelevant topics with low associated tfisf values. In contrast, when $\delta$ exceeds $0.9$, some relevant topics are overlooked and the quality of the summaries drops severely. Regarding parameter $\lambda$, we observe in graphs \ref{lambdaRouge2} and \ref{lambdaRougeSU4} that $\lambda=0.4$ yields the highest score since it combines both the relevance of themes to the query and their centrality within the corpus for the computation of hyperedge weights. In contrast, with $\lambda=1$, the algorithm focuses on the lexical similarity of themes with the query but it neglects the prominence of each theme. 

\begin{figure}[!h]
\centering
\subfigure{
    \includegraphics[width=.45\textwidth]{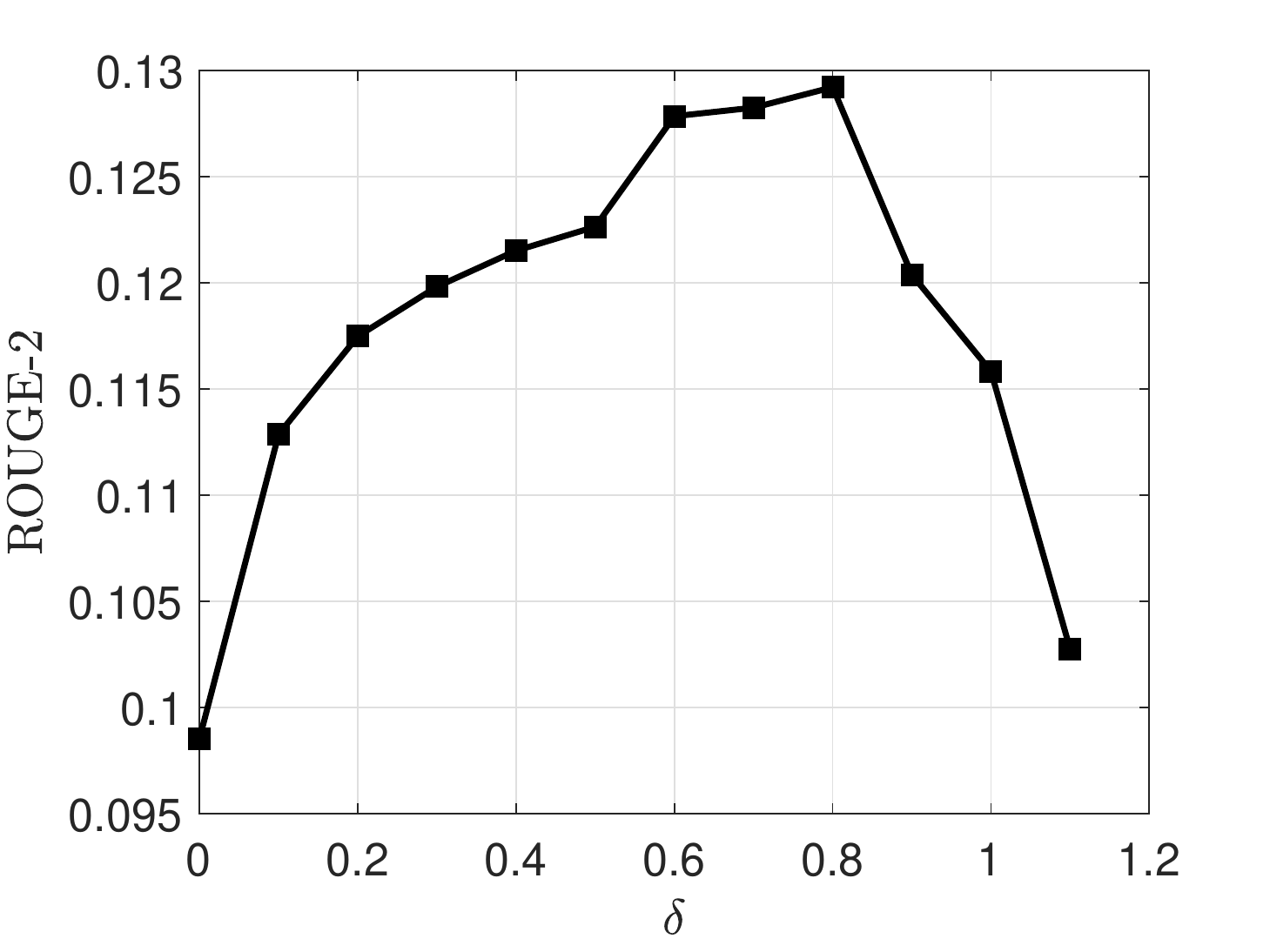}
    \label{deltaRouge2}
}
\subfigure{
	\includegraphics[width=.45\textwidth]{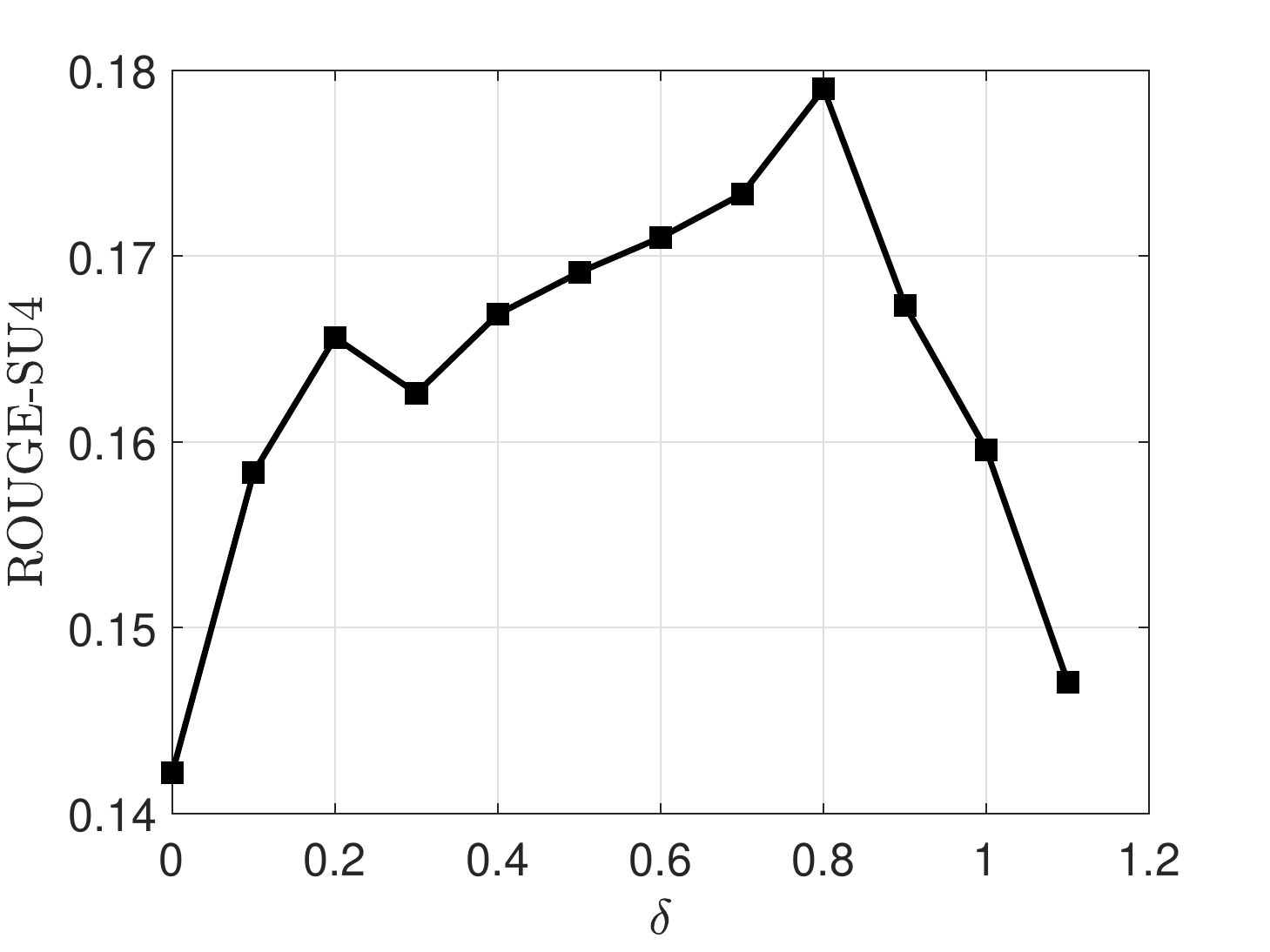}
	\label{deltaRougeSU4}
}
\caption{ROUGE-2 and ROUGE-SU4 as a function of $\delta$ for $\lambda=0.4$ and $\mu=1.98$.}
\label{deltaGraphs}
\end{figure}

\begin{figure}[!h]
\centering
\subfigure{
    \includegraphics[width=.45\textwidth]{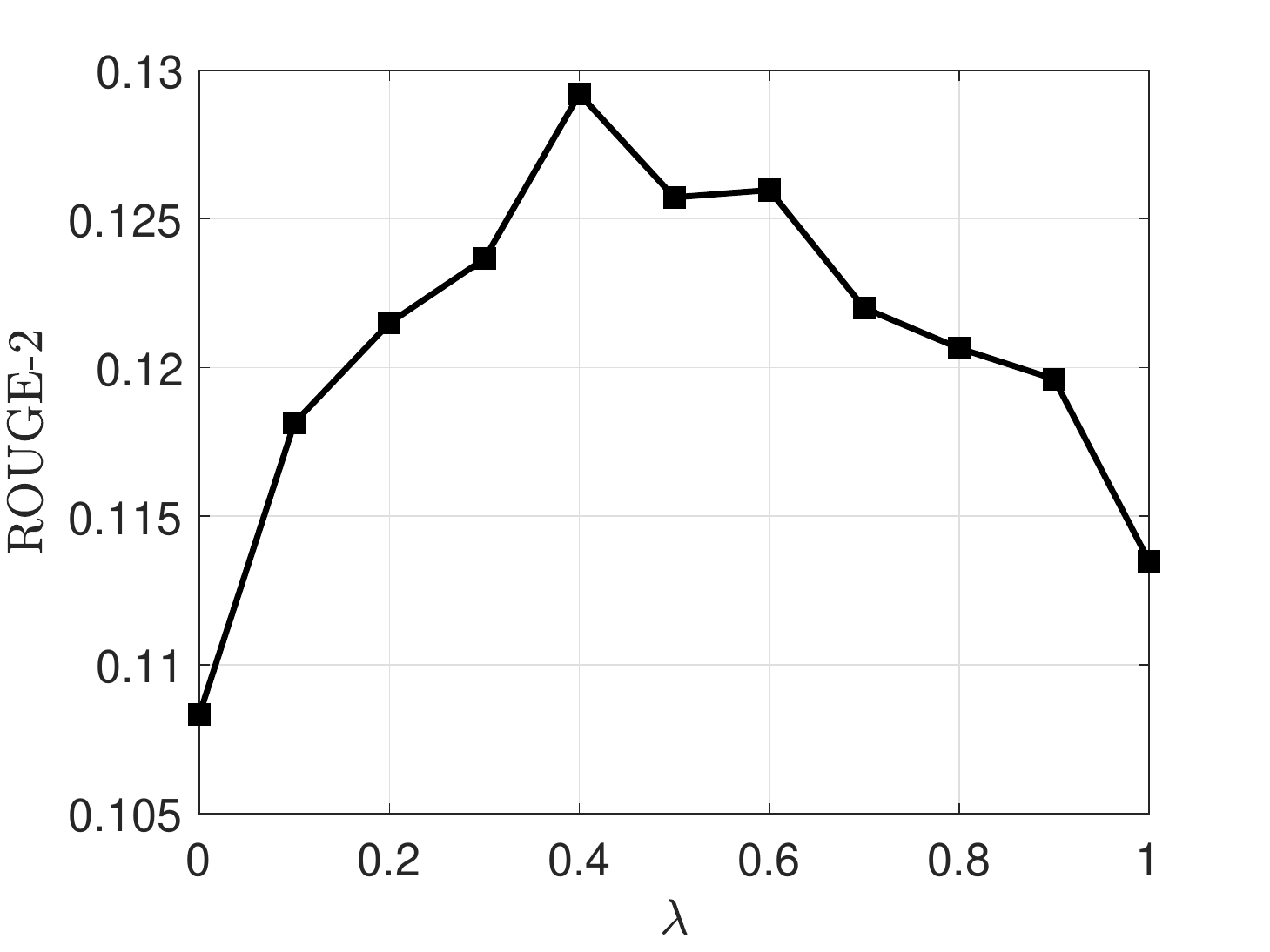}
    \label{lambdaRouge2}
}
\subfigure{
	\includegraphics[width=.45\textwidth]{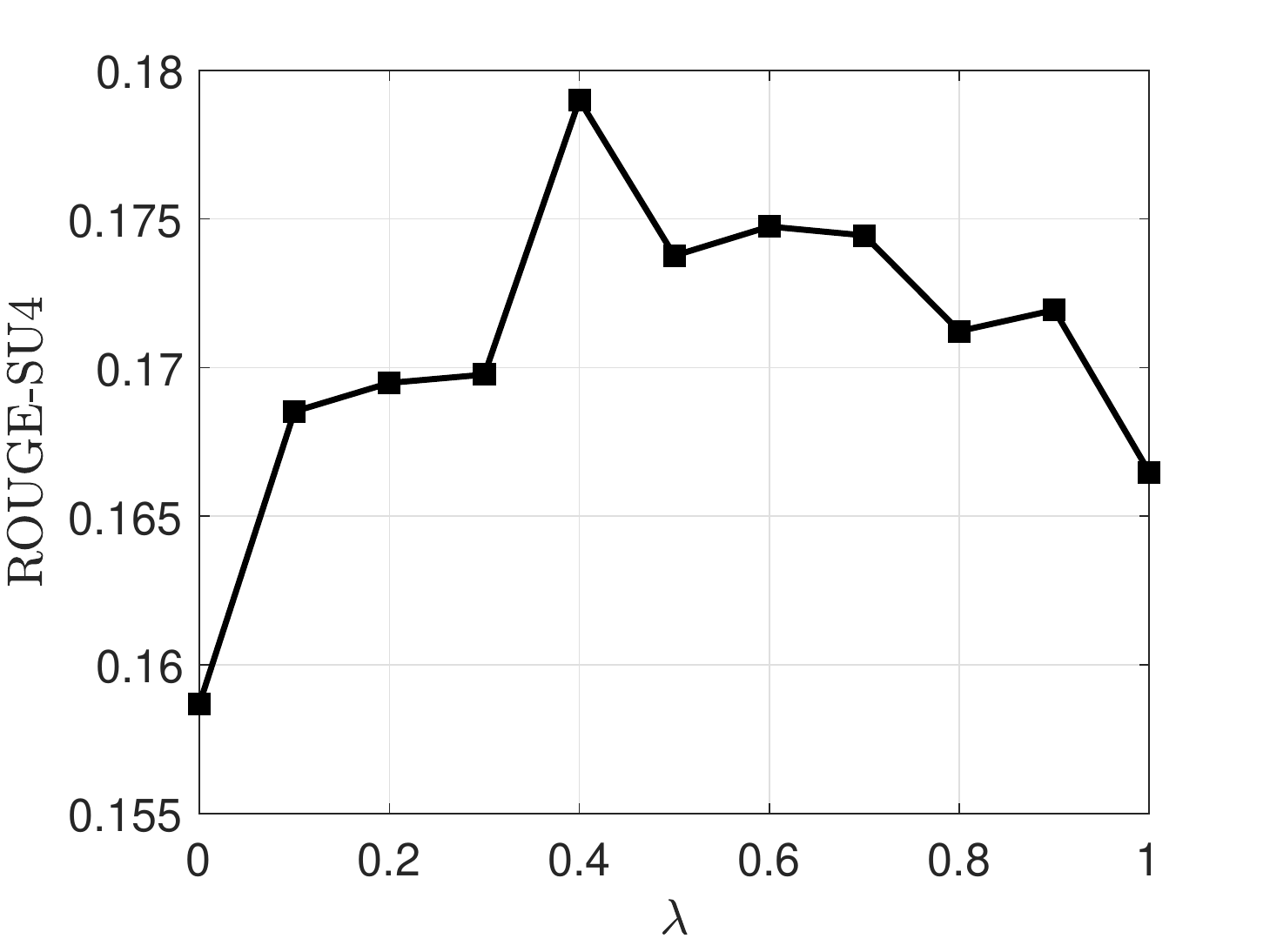}
	\label{lambdaRougeSU4}
}
\caption{ROUGE-2 and ROUGE-SU4 as a function of $\lambda$ for $\delta=0.85$ and $\mu=1.98$.}
\label{lambdaGraphs}
\end{figure}

\subsection{Testing the TC-TranSum algorithm}
In order to test our soft transversal-based summarizer, we display the evolution of the summary length and the ROUGE-SU4 score as a function of parameter $\gamma$ of algorithm \ref{softTransAlgo}. In figure \ref{gammaLENGTH}, we observe that the summary length grows linearly with the value of parameter $\gamma$ which confirms that our system does not favor longer sentences for low values of $\gamma$. The ROUGE-SU4 curve of figure \ref{gammaRSU} has a concave shape, with a low score when $\gamma$ is close to $0$ (due to a poor recall) or when $\gamma$ is close to $1$ (due to a poor precision). The overall concave shape of the ROUGE-SU4 curve also demonstrates the efficiency of our TC-TranSum algorithm: based on our hyperedge weighting scheme and our hyperedge coverage function, highly relevant sentences inducing a significant increase in the ROUGE-SU4 score are identified and included first in the summary.

In the subsequent experiments, we focus on TL-TranSum algorithm \ref{budgTransAlgo} which includes a target summary length and can thus be compared with other summarization systems which generally include a length constraint.

\begin{figure}[!h]
\centering
\subfigure{
    \includegraphics[width=.45\textwidth]{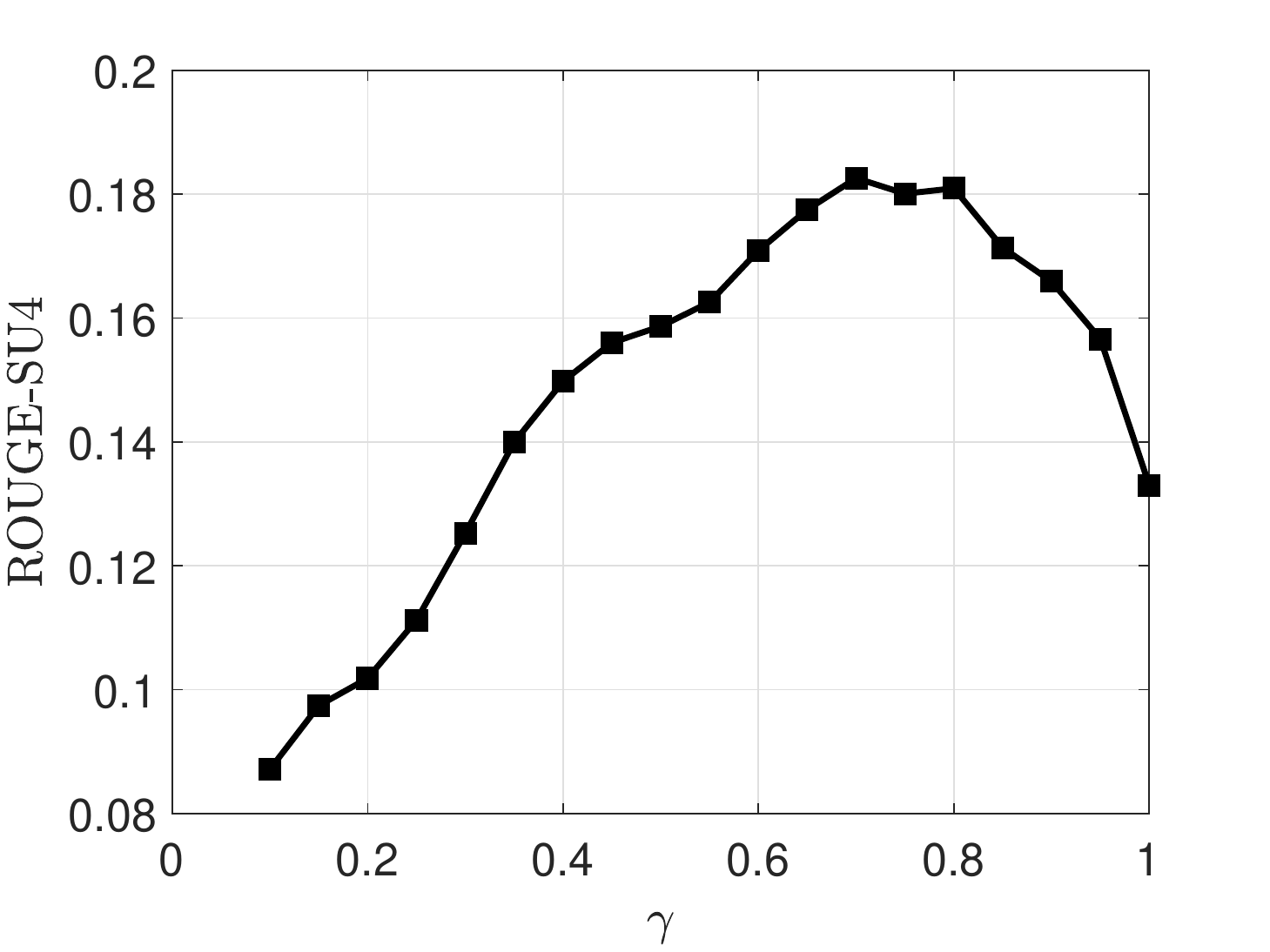}
    \label{gammaRSU}
}
\subfigure{
	\includegraphics[width=.45\textwidth]{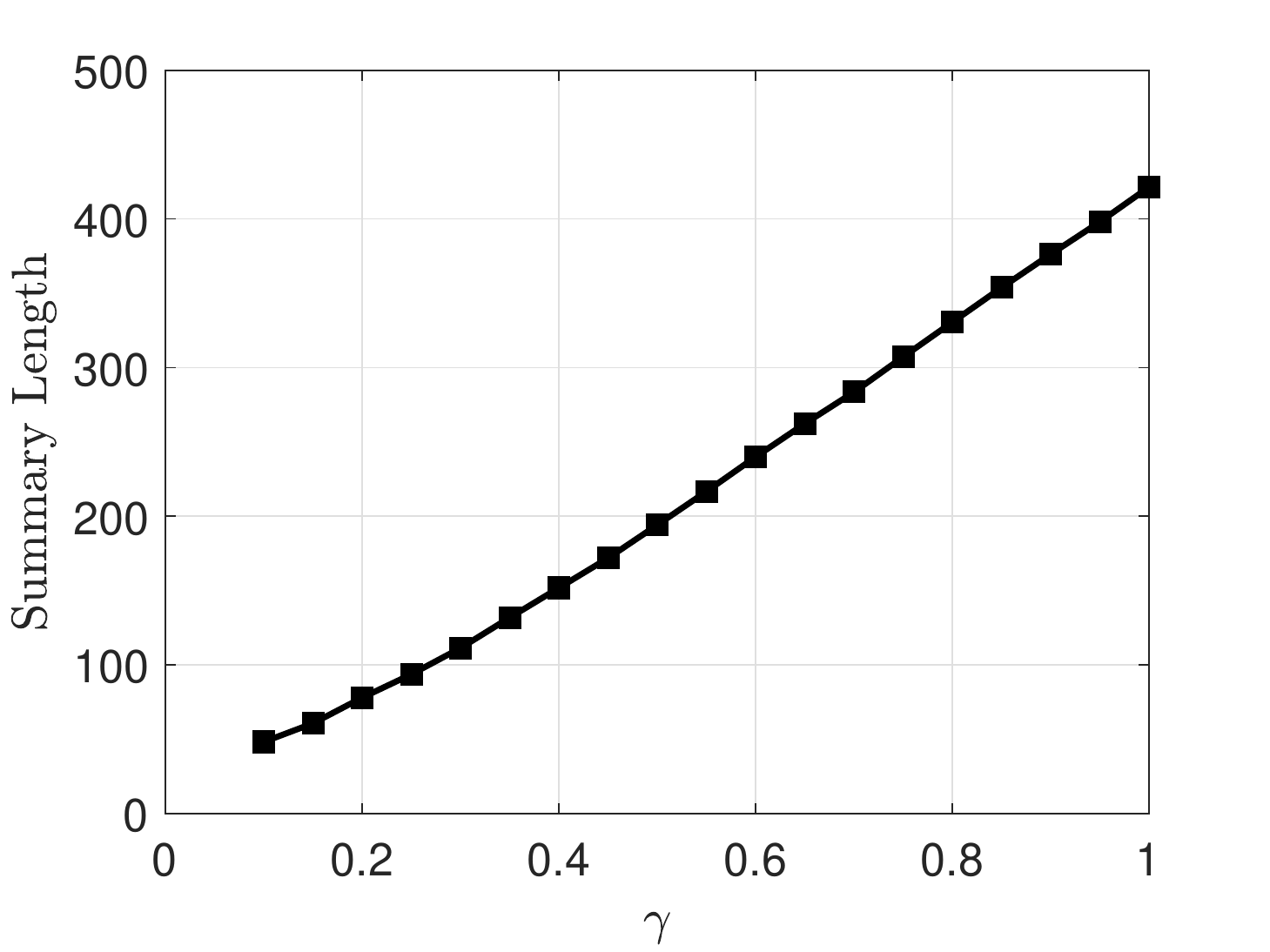}
	\label{gammaLENGTH}
}
\caption{Evolution of the ROUGE-SU4 score (left) and the summary length (right) as a function of the coverage parameter $\gamma$ of TC-TranSum algorithm \ref{softTransAlgo}.}
\label{otherAlgoGraphs}
\end{figure}

\subsection{Testing the hypergraph structure}\label{alternativeHypergraph}
To justify our theme-based hypergraph definition, we test other hypergraph models. We only change the hyperedge model which determines the kind of relationship between sentences that is captured by the hypergraph. The sentence selection is performed by applying algorithm \ref{budgTransAlgo} to the resulting hypergraph. We test three alternative hyperedge models. First a model based on \textit{agglomerative} clustering instead of SEMCOT: the same definition of semantic dissimilarity (equation \ref{semDistEqn}) is used, then topics are detected as clusters of terms obtained by agglomerative clustering with single linkage with the semantic dissimilarity as a distance measure. The themes are detected and the hypergraph is constructed in the same way as in our model. Second, \textit{Overlap} model defines hyperedges as overlapping clusters of sentences obtained by applying an algorithm of overlapping cluster detection \cite{lancichinetti2009} and using the cosine distance between tfisf representations of sentences as a distance metric. Finally, we test a hypergraph model already proposed in HyperSum system by \cite{hypersum} which combines pairwise hyperedges joining any two sentences having terms in common and hyperedges formed by non-overlapping clusters of sentences obtained by DBSCAN algorithm. Table \ref{tabComparisonGraphHyperModels} displays the ROUGE-2 and ROUGE-SU4 scores and the corresponding $95\%$ confidence intervals for each model. We observe that our model outperforms both \textit{HyperSum} and \textit{Overlap} models by at least $4\%$ and $15\%$ of ROUGE-SU4 score, respectively, which confirms that a two-step process extracting consistent topics first and then defining theme-based hyperedges from topic tags outperforms approaches based on sentence clustering, even when these clusters do overlap. Our model also outperforms the \textit{Agglomerative} model by $10\%$ of ROUGE-SU4 score, due to its ability to identify noisy terms and to detect the number of topics automatically.

\begin{table}[!h]
\begin{center}
\resizebox{\textwidth}{!}{\fontsize{7}{5.5}\selectfont
\begin{tabular}{c c c}
   \hline
   \rule{0pt}{2ex} System & ROUGE-2 & ROUGE-SU4 \\
   \hline
   
\textbf{TL-TranSum}	& $\mathbf{0.12997} \mathbf{(0.12548-0.13446)}$ & $\mathbf{0.17995} \mathbf{(0.17612-0.18377)}$ \\
Agglomerative & $0.12334 (0.11673-0.12994)$ &  $0.16292 (0.15302-0.17282)$\\
Overlap & $0.11831 (0.11334-0.12328)$ & $0.15640 (0.14762-0.16518)$ \\
HyperSum & $0.12317 (0.11743-0.12892)$ & $0.17231 (0.16561-0.17900)$ \\
\hline
\end{tabular}}
\end{center}
\caption{ROUGE-2 and ROUGE-SU4 scores for our TL-TranSum system compared to three other hypergraph models.}
\label{tabComparisonGraphHyperModels}
\end{table}
\FloatBarrier

\subsection{Comparison with related systems}
We compare the performance of our TL-TranSum algorithm \ref{budgTransAlgo} with that of five related summarization systems. \textit{Topic-sensitive LexRank} of \cite{R17} (TS-LexRank) and \textit{HITS} algorithms of \cite{hits} are early graph-based summarizers. TS-LexRank builds a sentence graph based on term co-occurrences in sentences, and it applies a query-biased PageRank algorithm for sentence scoring. HITS method additionally extracts clusters of sentences and it applies the hubs and authorities algorithm for sentence scoring, with the sentences as authorities and the clusters as hubs. As suggested in \cite{hypersum}, in order to extract query relevant sentences, only the top $5\%$ of sentences that are most relevant to the query are considered. \textit{HyperSum} extends early graph-based summarizers by defining a cluster-based hypergraph with the sentences as nodes and hyperedges as sentence clusters, as described in section \ref{alternativeHypergraph}. The sentences are then scored using an iterative label propagation algorithm over the hypergraph, starting with the lexical similarity of each sentence with the query as initial labels. In all three methods, the sentences with highest scores and pairwise lexical similarity not exceeding a threshold are included in the summary. Finally, we test two methods that also build on the theory of submodular functions. First, the \textit{MaxCover} approach \cite{takamura2009} seeks a summary by maximizing the number of distinct relevant terms appearing in the summary while not exceeding the target summary length (using equation \ref{eqnHyperWeights} to compute the term relevance scores). While the objective function of the method is similar to that of the problem of finding a maximal budgeted hypergraph transversal (equation \ref{defBudgTransv}) of \cite{wenpeng2015}, they overlook the semantic similarities between terms which are captured by our SEMCOT algorithm and our hypergraph model. Similarly, the \textit{Maximal Relevance Minimal Redundancy} (MRMR) first computes relevance scores of sentences as in equation \ref{eqnHyperWeights}, then it seeks a summary with a maximal total relevance score and a minimal redundancy while not exceeding the target summary length. The problem is solved by an iterative algorithm building on the submodularity and non-decreasing property of the objective function. 

Table \ref{tabComparisonGraphModels} displays the ROUGE-2 and ROUGE-SU4 scores with the corresponding $95\%$ confidence intervals for all six systems, including our TL-TranSum method. We observe that our system outperforms other graph and hypergraph-based summarizers involving the computation of individual sentence scores: LexRank by $6\%$, HITS by $13\%$ and HyperSum by $6\%$ of ROUGE-SU4 score; which confirms both the relevance of our theme-based hypergraph model and the capacity of our transversal-based summarizer to identify jointly relevant sentences as opposed to methods based on the computation of individual sentence scores. Moreover, our TL-TranSum method also outperforms other approaches such as MaxCover ($5\%$) and MRMR ($7\%$). These methods are also based on a submodular and non-decreasing function expressing the information coverage of the summary, but they are limited to lexical similarities between sentences and fail to detect topics and themes to measure the information coverage of the summary.

\begin{table}[!h]
\begin{center}
\resizebox{\textwidth}{!}{\fontsize{7}{5.5}\selectfont
\begin{tabular}{c c c}
   \hline
   \rule{0pt}{2ex} System & ROUGE-2 & ROUGE-SU4 \\
   \hline
   
\textbf{TL-TranSum}	& $\mathbf{0.12997} \mathbf{(0.12548-0.13446)}$ & $\mathbf{0.17995} \mathbf{(0.17612-0.18377)}$ \\
TS-LexRank	&  $0.11037 (0.10263-0.11811)$ &  $0.16939 (0.16233-0.17645)$\\
HITS & $0.10972 (0.10155-0.11789)$ & $0.15927 (0.15251-0.16603)$ \\
HyperSum &	 $0.11994 (0.11298-0.12690)$ & $0.16993 (0.16189-0.17797)$ \\
MaxCover	& $0.11985 (0.11028-0.12943)$ & $0.17072 (0.16155-0.17988)$ \\
MRMR & $0.11840 (0.10999-0.12681)$ & $0.16857 (0.16046-0.17668)$ \\

\hline
\end{tabular}}
\end{center}
\caption{Comparison with related graph- and hypergraph-based summarization systems.}
\label{tabComparisonGraphModels}
\end{table}
\FloatBarrier

\subsection{Comparison with DUC systems}

As a final experiment, we compare our TL-TranSum approach to other summarizers presented at DUC contests. Table \ref{tabComparDUC2007} displays the ROUGE-2 and ROUGE-SU4 scores for the worst summary produced by a human, for the top four systems submitted for the contests, for the baseline proposed by NIST (a summary consisting of the leading sentences of randomly selected documents) and the average score of all methods submitted, respectively for DUC2005, DUC2006 and DUC2007 contests. Regarding DUC2007, our method outperforms the best system by $2\%$ and the average ROUGE-SU4 score by $21\%$. It also performs significantly better than the baseline of NIST. However, it is outperformed by the human summarizer since our systems produces extracts, while humans naturally reformulate the original sentences to compress their content and produce more informative summaries. Tests on DUC2006 dataset lead to similar conclusions, with our TL-TranSum algorithm outperforming the best other system and the average ROUGE-SU4 score by $2\%$ and $22\%$, respectively. On DUC2005 dataset however, our TL-TranSum method is outperformed by the beset system which is due to the use of advanced NLP techniques (such as sentence trimming \cite{duc2005advanced}) which tend to increase the ROUGE-SU4 score. Nevertheless, the ROUGE-SU4 score produced by our TL-TranSum algorithm is still $15\%$ higher than the average score for DUC2005 contest.

\begin{table}[!h]
\begin{center}
\resizebox{\textwidth}{!}{\fontsize{10}{8}\selectfont
\begin{tabular}{c c c c c c c}
   \hline
	\rule{0pt}{2ex}  & \multicolumn{2}{c}{DUC2005} & \multicolumn{2}{c}{DUC2006} & \multicolumn{2}{c}{DUC2007}\\  
   \hline
   \rule{0pt}{2ex} Method & ROUGE-2& ROUGE-SU4 & ROUGE-2& ROUGE-SU4 & ROUGE-2& ROUGE-SU4\\
   \hline
	\rule{0pt}{2ex} Hum & $0.0897$ & $0.151$ & $0.13260$ & $0.18385$ & $0.17528$ & $0.21892$\\
	\rule{0pt}{2ex} \textbf{TL-TranSum} & $\mathbf{0.077392}$ & $0.12869$& $\mathbf{0.10779}$ & $\mathbf{0.15854}$& $\mathbf{0.12997}$ & $\mathbf{0.17995}$\\
1st & $0.07251$ & $\mathbf{0.13163}$ & $0.09558$ & $0.15529$ & $0.12448$ & $0.17711$\\
2nd & $0.07174$ & $0.12972$ & $0.09097$ & $0.14733$ & $0.12028$ & $0.17074$\\
3rd & $0.06984$ & $0.12525$ & $0.08987$ & $0.14755$ & $0.11887$& $0.16999$\\
4th & $0.06963$ & $0.12795$ & $0.08954$ & $0.14607$ & $0.11793$ & $0.17593$\\
Syst. Av.& $0.05842$ & $0.11205$ & $0.07463$ & $0.13021$ & $0.09597$&$0.14884$\\
Basel. & $0.04026$ & $0.08716$ & $0.04947$ & $0.09788$ & $0.06039$&$0.10507$\\
\hline
\end{tabular}}
\end{center}
\caption{Comparison with DUC2005, DUC2006 and DUC2007 systems}
\label{tabComparDUC2007}
\end{table}
\FloatBarrier

\section{Conclusion}\label{concludingSection}
In this paper, a new hypergraph-based summarization model was proposed, in which the nodes are the sentences of the corpus and the hyperedges are themes grouping sentences covering the same topics. Going beyond existing methods based on simple graphs and pairwise lexical similarities, our hypergraph model captures groups of semantically related sentences. Moreover, two new method of sentence selection based on the detection of hypergraph transversals were proposed: one to generate summaries of minimal length and achieving a target coverage, and the other to generate a summary achieving a maximal coverage of relevant themes while not exceeding a target length. The approach generates informative summaries by extracting a subset of sentences jointly covering the relevant themes of the corpus. Experiments on a real-world dataset demonstrate the effectiveness of the approach. The hypergraph model itself is shown to produce more accurate summaries than other models based on term or sentence clustering. The overall system also outperforms related graph- or hypergraph-based approaches by at least $10\%$ of ROUGE-SU4 score.

As a future research direction, we may analyse the performance of other algorithms for the detection of hypergraph transversals, such as methods based on LP relaxations. We may also further extend our topic model to take the polysemy of terms into acount: since each term may carry multiple meanings, a given term could refer to different topics depending on its context. Finally, we intend to adapt our model for solving related problems, such as commmunity question answering.


\begin{thebibliography}{10}

\bibitem{lexrank}
G.~Erkan and D.~R. Radev, ``Lexrank: Graph-based lexical centrality as salience
  in text summarization,'' {\em Journal of Artificial Intelligence Research},
  vol.~22, pp.~457--479, 2004.

\bibitem{hits}
X.~Wan and J.~Yang, ``Multi-document summarization using cluster-based link
  analysis,'' in {\em Proceedings of the 31st annual international ACM SIGIR
  conference on Research and development in information retrieval},
  pp.~299--306, ACM, 2008.

\bibitem{R17}
J.~Otterbacher, G.~Erkan, and D.~R. Radev, ``Using random walks for
  question-focused sentence retrieval,'' in {\em Proceedings of the conference
  on Human Language Technology and Empirical Methods in Natural Language
  Processing}, pp.~915--922, ACL, 2005.

\bibitem{hypersum}
W.~Wang, S.~Li, J.~Li, W.~Li, and F.~Wei, ``Exploring hypergraph-based
  semi-supervised ranking for query-oriented summarization,'' {\em Information
  Sciences}, vol.~237, pp.~271--286, 2013.

\bibitem{herf}
S.~Xiong and D.~Ji, ``Query-focused multi-document summarization using
  hypergraph-based ranking,'' {\em Information Processing \& Management},
  vol.~52, no.~4, pp.~670--681, 2016.

\bibitem{gunopulous1997}
D.~Gunopulos, H.~Mannila, R.~Khardon, and H.~Toivonen, ``Data mining,
  hypergraph transversals, and machine learning,'' in {\em Proceedings of the
  sixteenth ACM SIGACT-SIGMOD-SIGART symposium on Principles of database
  systems}, pp.~209--216, ACM, 1997.

\bibitem{klamt2009}
S.~Klamt, U.~U. Haus, and F.~Theis, ``Hypergraphs and cellular networks,'' {\em
  PLoS computational biology}, vol.~5, no.~5, p.~e1000385, 2009.

\bibitem{hong2014repository}
K.~Hong, J.~M. Conroy, B.~Favre, A.~Kulesza, H.~Lin, and A.~Nenkova, ``A
  repository of state of the art and competitive baseline summaries for generic
  news summarization.,'' in {\em LREC}, pp.~1608--1616, 2014.

\bibitem{kanapala2017text}
A.~Kanapala, S.~Pal, and R.~Pamula, ``Text summarization from legal documents:
  a survey,'' {\em Artificial Intelligence Review}, pp.~1--32, 2017.

\bibitem{takamura2009}
H.~Takamura and M.~Okumura, ``Text summarization model based on maximum
  coverage problem and its variant,'' in {\em Proceedings of the 12th
  Conference of the European Chapter of the Association for Computational
  Linguistics}, pp.~781--789, Association for Computational Linguistics, 2009.

\bibitem{mckeown2011}
A.~Nenkova, K.~McKeown, {\em et~al.}, ``Automatic summarization,'' {\em
  Foundations and Trends{\textregistered} in Information Retrieval}, vol.~5,
  no.~2--3, pp.~103--233, 2011.

\bibitem{nenkova2012}
A.~Nenkova and K.~McKeown, ``A survey of text summarization techniques,'' in
  {\em Mining text data} (C.~C. Aggarwal and C.~Zhai, eds.), ch.~3, pp.~43--76,
  Springer Science \& Business Media, 2012.

\bibitem{fattah2014}
M.~A. Fattah, ``A hybrid machine learning model for multi-document
  summarization,'' {\em Applied intelligence}, vol.~40, no.~4, pp.~592--600,
  2014.

\bibitem{zha2002}
H.~Zha, ``Generic summarization and keyphrase extraction using mutual
  reinforcement principle and sentence clustering,'' in {\em Proceedings of the
  25th annual international ACM SIGIR conference on Research and development in
  information retrieval}, pp.~113--120, ACM, 2002.

\bibitem{zhang2012}
Z.~Zhang, S.~S. Ge, and H.~He, ``Mutual-reinforcement document summarization
  using embedded graph based sentence clustering for storytelling,'' {\em
  Information Processing \& Management}, vol.~48, no.~4, pp.~767--778, 2012.

\bibitem{bilmes2010}
H.~Lin and J.~Bilmes, ``Multi-document summarization via budgeted maximization
  of submodular functions,'' in {\em Human Language Technologies: The 2010
  Annual Conference of the North American Chapter of the Association for
  Computational Linguistics}, pp.~912--920, Association for Computational
  Linguistics, 2010.

\bibitem{wenpeng2015}
W.~Yin and Y.~Pei, ``Optimizing sentence modeling and selection for document
  summarization,'' in {\em Proceedings of the Twenty-Fourth International Joint
  Conference on Artificial Intelligence}, pp.~1383--1389, 2015.

\bibitem{shen2010}
C.~Shen and T.~Li, ``Multi-document summarization via the minimum dominating
  set,'' in {\em Proceedings of the 23rd International Conference on
  Computational Linguistics}, pp.~984--992, Association for Computational
  Linguistics, 2010.

\bibitem{andrew2017}
A.~Gainer-Dewar and P.~Vera-Licona, ``The minimal hitting set generation
  problem: algorithms and computation,'' {\em SIAM Journal on Discrete
  Mathematics}, vol.~31, no.~1, pp.~63--100, 2017.

\bibitem{boros}
E.~Boros, V.~Gurvich, L.~Khachiyan, and K.~Makino, ``Dual-bounded generating
  problems: weighted transversals of a hypergraph,'' {\em Discrete Applied
  Mathematics}, vol.~142, no.~1, pp.~1--15, 2004.

\bibitem{wolsey1982}
L.~A. Wolsey, ``An analysis of the greedy algorithm for the submodular set
  covering problem,'' {\em Combinatorica}, vol.~2, no.~4, pp.~385--393, 1982.

\bibitem{gainerdewar2017}
A.~Gainer-Dewar and P.~Vera-Licona, ``The minimal hitting set generation
  problem: algorithms and computation,'' {\em SIAM Journal on Discrete
  Mathematics}, vol.~31, no.~1, pp.~63--100, 2017.

\bibitem{stemmer}
M.~F. Porter, ``Snowball: A language for stemming algorithms.''
  \url{http://www.snowball.tartarus.org/texts/introduction.html}, 2001.
\newblock Accessed 15 November 2017.

\bibitem{KBS1}
G.~Salton and C.~Buckley, ``Term-weighting approaches in automatic text
  retrieval,'' {\em Information processing \& management}, vol.~24, no.~5,
  pp.~513--523, 1988.

\bibitem{moitra2012}
S.~Arora, R.~Ge, and A.~Moitra, ``Learning topic models--going beyond svd,'' in
  {\em Foundations of Computer Science (FOCS), 2012 IEEE 53rd Annual Symposium
  on}, pp.~1--10, IEEE, 2012.

\bibitem{cilibrasi2007}
R.~L. Cilibrasi and P.~M. Vitanyi, ``The google similarity distance,'' {\em
  IEEE Transactions on knowledge and data engineering}, vol.~19, no.~3, 2007.

\bibitem{cilibrasi2005clustering}
R.~Cilibrasi and P.~M. Vit{\'a}nyi, ``Clustering by compression,'' {\em IEEE
  Transactions on Information theory}, vol.~51, no.~4, pp.~1523--1545, 2005.

\bibitem{dbscan}
M.~Ester, H.-P. Kriegel, J.~Sander, X.~Xu, {\em et~al.}, ``A density-based
  algorithm for discovering clusters in large spatial databases with noise,''
  in {\em KDD'96 Proceedings of the Second International Conference on
  Knowledge Discovery and Data Mining}, vol.~96, pp.~226--231, 1996.

\bibitem{maimon2005}
L.~Rokach and O.~Maimon, ``Clustering methods,'' in {\em Data mining and
  knowledge discovery handbook}, pp.~321--352, Springer, 2005.

\bibitem{karp1972}
R.~M. Karp, ``Reducibility among combinatorial problems,'' in {\em Complexity
  of computer computations} (R.~Miller, ed.), pp.~85--103, Springer, 1972.

\bibitem{leskovec2007}
J.~Leskovec, A.~Krause, C.~Guestrin, C.~Faloutsos, J.~VanBriesen, and
  N.~Glance, ``Cost-effective outbreak detection in networks,'' in {\em
  Proceedings of the 13th ACM SIGKDD international conference on Knowledge
  discovery and data mining}, pp.~420--429, ACM, 2007.

\bibitem{carbonell1998}
J.~Carbonell and J.~Goldstein, ``The use of mmr, diversity-based reranking for
  reordering documents and producing summaries,'' in {\em Proceedings of the
  21st annual international ACM SIGIR conference on Research and development in
  information retrieval}, pp.~335--336, ACM, 1998.

\bibitem{duc2005}
H.~T. Dang, ``Overview of duc 2005,'' in {\em Proceedings of the document
  understanding conference}, 2005.

\bibitem{duc2006}
T.~D. Hoa, ``Overview of duc 2006,'' in {\em Proceedings of the document
  understanding conference}, 2006.

\bibitem{duc2007}
H.~T. Dang, ``Overview of the duc 2007 summarization task,'' in {\em
  Proceedings of the document understanding conference}, 2007.

\bibitem{rouge}
C.-Y. Lin and E.~Hovy, ``Automatic evaluation of summaries using n-gram
  co-occurrence statistics,'' in {\em Proceedings of the 2003 Conference of the
  North American Chapter of the Association for Computational Linguistics on
  Human Language Technology-Volume 1}, pp.~71--78, Association for
  Computational Linguistics, 2003.

\bibitem{lancichinetti2009}
A.~Lancichinetti, S.~Fortunato, and J.~Kert{\'e}sz, ``Detecting the overlapping
  and hierarchical community structure in complex networks,'' {\em New Journal
  of Physics}, vol.~11, no.~3, p.~033015, 2009.

\bibitem{duc2005advanced}
D.~Zajic, B.~Dorr, R.~Schwartz, C.~Monz, and J.~Lin, ``A sentence-trimming
  approach to multi-document summarization,'' in {\em Proceedings of HLT/EMNLP
  2005 Workshop on Text Summarization (HLT/EMNLP 05)}, pp.~151--158, 2005.

\end{thebibliography}


\section*{References}
\end{document}